\documentclass[sigconf, nonacm]{acmart}
\AtBeginDocument{%
  }

\setcopyright{acmlicensed}
\copyrightyear{2018}
\acmYear{2018}
\acmDOI{XXXXXXX.XXXXXXX}
\acmConference[Conference acronym 'XX]{Make sure to enter the correct
  conference title from your rights confirmation email}{June 03--05,
  2018}{Woodstock, NY}
\acmISBN{978-1-4503-XXXX-X/2018/06}

\usepackage[T1]{fontenc}
\usepackage[utf8]{inputenc} 
\usepackage{amsfonts}
\usepackage{amsthm}
\usepackage{mathtools}
\usepackage{bm}
\usepackage{bbm}
\usepackage{dsfont}
\usepackage{mathrsfs}
\usepackage{graphicx}
\usepackage{subfig} 
\usepackage{booktabs}
\usepackage{multirow}
\usepackage{makecell}
\usepackage{diagbox}
\usepackage{adjustbox}
\usepackage{threeparttable}
\usepackage{caption}
\usepackage[ruled,linesnumbered,noend]{algorithm2e}
\usepackage{enumitem}
\usepackage{multicol}
\usepackage{balance}
\usepackage{float}
\usepackage{cuted}
\usepackage{xspace}
\usepackage{textcomp}
\usepackage{nicefrac}
\usepackage{microtype}
\usepackage{pifont}
\usepackage[normalem]{ulem}
\usepackage[dvipsnames]{xcolor}
\usepackage{numprint}
\usepackage{siunitx}
\usepackage[most]{tcolorbox}
\usepackage{appendix}
\usepackage[hang,flushmargin]{footmisc}
\usepackage{url}
\usepackage{hyperref}

\PassOptionsToPackage{square,numbers}{natbib}

\floatname{algorithm}{Function}
\floatname{algorithm}{Procedure}

\newtheorem{theorem}{Theorem}

\newtheorem{definition}{Definition}
\npdecimalsign{.}

\begin{document}
\title{Alignment-Guided Largest Table Overlap Size Estimation}

\author{Ge Lee}
\authornote{This work was done while the author was a visiting student at The University of Queensland and affiliated with Data61, CSIRO.}
\affiliation{%
  \institution{RMIT University}
  \city{Melbourne}
  \country{Australia}}
\email{ge.lee@student.rmit.edu.au}

\author{Shixun Huang}
\affiliation{%
  \institution{University of Wollongong}
  \city{Wollongong}
  \country{Australia}}
\email{shixunh@uow.edu.au}

\author{Zhifeng Bao}
\affiliation{%
  \institution{The University of Queensland}
  \city{Brisbane}
  \country{Australia}}
\email{zhifeng.bao@uq.edu.au}

\author{Shazia Sadiq}
\affiliation{%
  \institution{The University of Queensland}
  \city{Brisbane}
  \country{Australia}}
\email{s.sadiq@uq.edu.au}

\author{Yanchang Zhao}
\affiliation{%
  \institution{Data61, CSIRO}
  \city{Canberra}
  \country{Australia}}
\email{yanchang.zhao@csiro.au}

\newcommand{\huang}[1]{{\color{purple}{\bf{Huang comment:}} #1}}
\newcommand{\bao}[1]{{\color{purple}{\bf{Bao comment:}} #1}}
\newcommand{\lee}[1]{{\color{purple}{\bf{Ge comment:}} #1}}

\newcommand{\todo}[1]{{\color{black} #1}}
\newcommand{\revised}[1]{{\color{black} #1}}

\newcommand{\plan}[1]{{\color{black} #1}}            
\newcommand{\revtwo}[1]{{\color{black} #1}}    
\newcommand{\revthree}[1]{{\color{black} #1}}   
\newcommand{\revfour}[1]{{\color{black} #1}}  

\newcommand{\commt}[1]{\textit{  // #1}}

\newcommand{\wiki}{Wiki}
\newcommand{\git}{Git}
\newcommand{\gitq}{Git-Query}
\newcommand{\real}{RealEstate}
\newcommand{\re}{RE}


\newcommand{\ec}{\textsc{Encode-Compare}}
\newcommand{\encode}{\textsc{Encode}}
\newcommand{\compare}{\textsc{Compare}}
\newcommand{\iep}{\textsc{Interact-Encode-Predict}}
\newcommand{\interact}{\textsc{Interact}}
\newcommand{\predict}{\textsc{Predict}}

\newcommand{\js}{\texttt{Jaccard}}
\newcommand{\jbu}{\texttt{Jaccard-BU}}
\newcommand{\jbn}{\texttt{Jaccard-BN}}

\newcommand{\bert}{\texttt{BERT}}
\newcommand{\bertR}{\texttt{BERT-R}}
\newcommand{\bertT}{\texttt{BERT-T}}
\newcommand{\bertHT}{\texttt{BERT-HT}}

\newcommand{\rob}{\texttt{RoBERTa}}
\newcommand{\robR}{\texttt{RoBERTa-R}}
\newcommand{\robT}{\texttt{RoBERTa-T}}
\newcommand{\robHT}{\texttt{RoBERTa-HT}}

\newcommand{\turl}{\texttt{TURL}}
\newcommand{\embdi}{\texttt{EmbDI}}

\newcommand{\sloth}{\texttt{Sloth}}
\newcommand{\arma}{\texttt{Armadillo}}
\newcommand{\our}{\texttt{ALORE}}

\newcommand{\Sloth}{Sloth}
\newcommand{\Arma}{Armadillo}
\newcommand{\Our}{ALORE}

\definecolor{bgBest}{HTML}{A2C4EF}   
\definecolor{bgSec}{HTML}{D9E8F8}    

\newcommand{\hlBest}[1]{{%
    \setlength{\fboxsep}{1pt}%
    \colorbox{bgBest}{\textbf{#1}}%
}}
\newcommand{\hlSec}[1]{{%
    \setlength{\fboxsep}{1pt}%
    \colorbox{bgSec}{#1}%
}}


\newtcolorbox{revbox}{
    colback=gray!10, 
    colframe=white, 
    boxrule=0pt, 
    arc=1.5mm, 
    left=1.5mm, 
    right=1.5mm, 
    top=0.7mm, 
    bottom=0.7mm, 
    enhanced, 
    breakable,
    after skip=0.5em
}

\newcommand{\sdval}[2]{#1\textsubscript{\ensuremath{\,\pm\,}#2}}
\newcommand{\sdvalpct}[2]{%
  \makebox[2.3em][r]{#1}%
  \textsubscript{\ensuremath{\,\pm\,}\makebox[1.8em][l]{#2}}%
}
\newcommand{\oot}{\textsc{oot}}
\begin{abstract}

Fast estimation of the size of the largest overlap between tables enables blocking and query-by-table retrieval in large table repositories.
The first and the state-of-the-art estimator \Arma{} improves efficiency by embedding each table independently and approximating overlap ratio via embedding similarity. However, accurate estimation in heterogeneous repositories remains limited by three challenges: (C1)~overlap depends on \revfour{row--column structure, i.e., each matched cell must preserve both its row and column membership under a joint alignment of the two tables}, but existing encodings leave this structure to be inferred indirectly; (C2)~independent encoding provides no explicit channel for inter-table alignment signals, biasing prediction toward global similarity; (C3)~naïve value encodings overfit to corpus-specific distributions, causing cross-domain degradation.
Hence, we propose \Our{}, a scalable and domain-robust overlap ratio estimator built on three principles: (P1)~explicitly represent row--column structure; (P2)~expose inter-table alignment signals during training without expensive alignment search; (P3)~reduce sensitivity to corpus-specific value distributions. \Our{} instantiates these principles with a Two-View Row--Column Hypergraph encoder, alignment-guided objectives with inexpensive interaction signals, and a domain-robust value mapping. Experiments on multiple datasets spanning diverse domains and scales, including a large real-world corpus beyond prior benchmarks, show that \Our{} outperforms the state of the art. \Our{} reduces MAE by up to 55\% overall and 69\% in zero-shot transfer, while achieving up to $89\times$ speedup. We further validate its effectiveness for query-by-table retrieval.

\end{abstract}

\maketitle

\section{Introduction}

We study \emph{table overlap ratio estimation}~\cite{pugnaloni2025table}: given two tables, estimate the \emph{size of their largest overlap}. Intuitively, this is the number of cells in the largest common rectangular subtable that the two tables can share exactly in value. This subtable is obtained by reordering rows and columns through injective (one-to-one) row and column alignments. We refer to this size, normalized by the number of cells in the smaller table, as the \emph{overlap ratio}.

\revfour{Consider a workflow of repository search and curation over tables of property sales. A user provides a target table of property sale records and asks the system to find duplicate or near-duplicate tables in a large repository collected from public agencies, listing portals, and archived snapshots. Most candidates are unrelated, even though many may appear superficially similar. Relevant candidates may contain the same sale records with rows and columns reordered, cover only a subset of the target because they were extracted over different time windows, or have missing and inconsistent headers. These candidates are difficult to identify from metadata alone, and exact computation over all candidates is too expensive. The system therefore needs a fast filtering step that retains likely reordered copies and partial extracts, while discarding unrelated tables before expensive exact verification.}

\revfour{The above illustrates why overlap ratio estimation is a useful primitive in large table repositories~\cite{pugnaloni2025table}. A fast estimator can reduce a very large candidate space to a small set of candidate tables that share substantial content.} These candidates can then be handled by more expensive downstream processing, such as exact overlap computation~\cite{zecchini2024determining}, table reclamation~\cite{fan2024gent}, deduplication~\cite{koch2023duplicate}, or related table discovery~\cite{das2012finding}. The same estimate supports query-by-table retrieval~\cite{bleifuss2021structured, bleifuss2021secret, hulsebos2023gittables}, where a user provides a query table and the system retrieves or filters tables by overlap ratio. It also facilitates versioning and evolution analysis by quantifying content changes across snapshots~\cite{zecchini2024determining, bleifuss2021structured}. Crucially, these settings cannot assume reliable schema cues, metadata, or pre-aligned row/column organization: headers are frequently inconsistent or missing~\cite{adelfio2013schema, cafarella2008webtables}, and about 20\% of Web tables lack identifiable headers~\cite{pimplikar2012answering}. Figure~\ref{fig:example} summarizes these three use cases.

\smallskip\noindent\textbf{Existing Solutions and Open Challenges.}
Computing exact overlap ratio requires finding the \emph{largest overlap}, a combinatorial row/column alignment problem that is NP-hard~\cite{zecchini2024determining}. \Sloth{}~\cite{zecchini2024determining} solves it exactly but is prohibitively slow at scale, \revthree{taking days} for 100k pairs and hours per query over 10k tables~\cite{zecchini2024determining,pugnaloni2025table}. 
Motivated by this, \Arma{}~\cite{pugnaloni2025table} is the first work that proposes a learning-based estimator and remains the state of the art. Essentially, it encodes each table as a graph, embeds them independently, and estimates overlap ratio from the similarity between the two embeddings. By replacing expensive alignment search with embedding similarity, \Arma{} improves efficiency drastically.

In practice, however, overlap ratio estimation must work in heterogeneous repositories. Tables come from diverse sources with different layouts, distributions, and noise. Obtaining overlap labels for new table pairs is expensive because it requires solving combinatorial alignments. Consequently, retraining for every new corpus is often impractical. This makes overlap ratio estimation challenging. We want an estimator to remain accurate not only in-domain, but also when deployed on unseen repositories, without requiring new labels. To meet \revfour{these requirements}, an estimator must (i)~represent the alignment-dependent \revfour{row--column structure} that defines overlap. \revfour{We use row--column structure to refer to the organization of cells by their row and column memberships, together with the requirement that valid overlap must preserve both memberships under a joint row and column alignment.} It must also (ii)~incorporate pair-specific correspondence signals without expensive alignment search, and (iii)~remain robust across repositories with different content. Motivated by this setting, we identify three open challenges:

\noindent$\bullet$ \textbf{C1: Standard graphs underexpress joint row--column structure.}
Overlap depends on a \emph{joint} row--column alignment, but \Arma{}’s graph connects rows and columns to cells through pairwise edges, leaving \revfour{this joint structure} to be inferred indirectly.

\noindent$\bullet$ \textbf{C2: Independent encoding biases the model toward global similarity.} Embedding tables independently provides no explicit channel to model which parts of $T_1$ should align to which parts~of~$T_2$.

\noindent$\bullet$ \textbf{C3: Value distribution shift causes cross-domain degradation.}
Value vocabularies and frequencies vary across corpora, and naïve encoding overfits to corpus-specific patterns (over $3.7\times$ higher error in zero-shot transfer in Section~\ref{subsec:exp-crossdomain}).

\begin{figure}[t]
    \centering
    \includegraphics[width=0.47\textwidth]{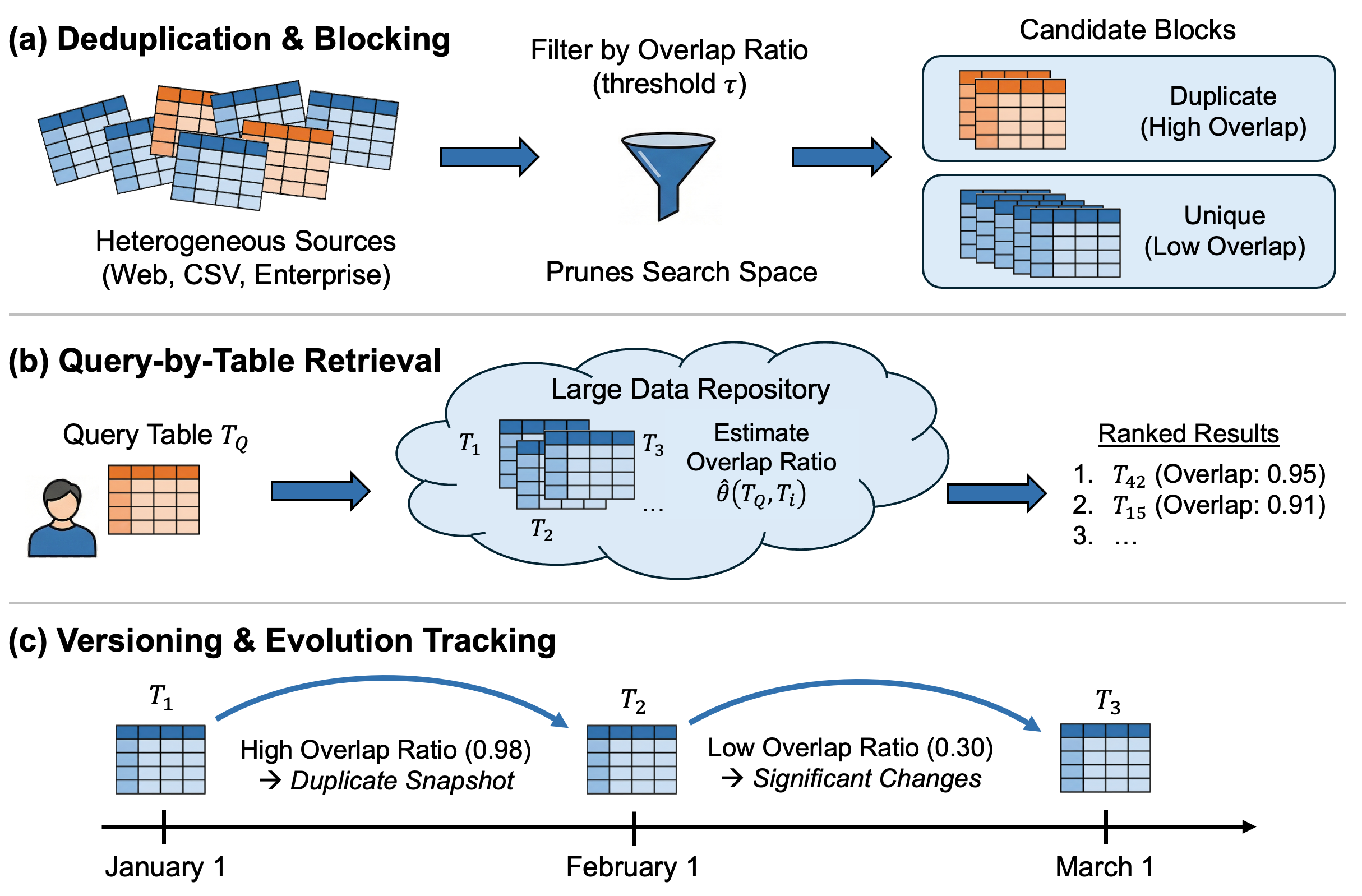}
    \caption[]{Use cases of overlap ratio estimation.}
    \label{fig:example}
\end{figure}

\smallskip\noindent\textbf{Our Solution.}
To address C1--C3, we propose \Our{} (\uline{AL}ignment-guided \uline{O}verlap \uline{R}atio \uline{E}stimator), an interaction-aware model for scalable and domain-robust overlap ratio estimation. \Our{} is designed around three principles:

\noindent$\bullet$ \textbf{P1: Structure-expressive encoding.}
To make \revfour{row--column structure} explicit (C1), we model each table as a Two-View Row--Column Hypergraph with row and column hyperedges over cell nodes. We then encode this  with a structure-aware encoder to produce embeddings that preserve these structural constraints.

\noindent$\bullet$ \textbf{P2: Inter-table alignment-guided learning.}
To expose correspondence signals that independent encoding misses (C2), we introduce a multi-granularity alignment regularizer that encourages row- and column-level correspondences during training, alongside the overlap regression loss. At prediction time, we fuse the learned embeddings with lightweight interaction features that capture pair-specific overlap signals, while keeping inference efficient. 

\noindent$\bullet$ \textbf{P3: Domain-robust value mapping.}
To improve zero-shot transfer under value distribution shift (C3), we map cell values into discrete buckets and train \Our{} to be invariant to random relabelings of bucket indices via a consistency regularizer. This reduces reliance on corpus-specific value frequencies while preserving exact value equality within each table pair.

\smallskip
Our contributions are as follows:
\begin{itemize}[leftmargin=*]
    \item We identify three challenges for overlap ratio estimation: underexpressed \revfour{row--column structure}, missing inter-table alignment signals under independent encoding, and value distribution shift across domains (Section~\ref{subsec:challenges}). These observations motivate three design principles that shape \Our{} (Section~\ref{subsec:principles}).
    
    \item We propose a Two-View Row--Column Hypergraph with a dedicated encoder that produces structure-preserving embeddings capturing joint \revfour{row--column structure} (Sections~\ref{subsec:mtd-hypergraph}~and~\ref{subsec:mtd-encoder}).
    
    \item We develop an \iep{} pipeline with alignment-guided objectives that incorporate inter-table interaction signals without requiring exact alignment: \interact{} computes pairwise signals, \encode{} produces hypergraph-based representations, and \predict{} fuses them to regress overlap ratio. (Section~\ref{subsec:mtd-pred}).
    
    \item We introduce stochastic bucket permutations and a consistency regularizer to improve domain robustness (Sections~\ref{subsec:mtd-valuemap} and~\ref{subsec:mtd-obj}).

    \item We provide theoretical analysis showing that \Our{} (i) is invariant to row and column permutations, as required by the overlap definition; (ii) captures full row and column context efficiently within each encoder layer; and (iii) provides domain robustness by enforcing invariance to index relabeling (Section~\ref{sec:theory}).

    \item We evaluate on three datasets spanning multiple domains and scales, including a large real-world corpus beyond prior benchmarks. Across in-domain and cross-domain settings, \Our{} consistently outperforms the state of the art, reducing MAE by up to 55\% overall and 69\% in zero-shot transfer, while achieving up to $89\times$ speedup. We further validate its effectiveness for query-by-table retrieval under both ranking and threshold-based retrieval (Section~\ref{sec:exp}).
\end{itemize}

\section{Preliminaries}\label{sec:preliminaries}

We begin by introducing the definitions of table overlap and overlap ratio, followed by the formal definition of the overlap ratio estimation task. Then, we briefly review the current state of the art.

\subsection{Problem Formulation}

A \emph{table} $T$ is a two-dimensional structure with $m$ rows and $n$ columns. Let $I = \{1, \dots, m\}$ and $J = \{1, \dots, n\}$ be the sets of row and column indices, respectively. The value stored in the cell located at row $i \in I$ and column $j \in J$ is denoted by $c_{i,\,j}$. The size of the table is $|T| = mn$. \revfour{For two tables $T_1$ and $T_2$, we use $c^{(1)}_{i,\,j}$ and $c^{(2)}_{i,\,j}$ to denote cell values in $T_1$ and $T_2$, respectively.}

\smallskip\noindent\textbf{Overlap Between Two Tables.}
A key property of tables is that their semantics are invariant to row and column order. Accordingly, overlap is defined with respect to an \emph{alignment} between their indices~\cite{zecchini2024determining}. Consider two tables $T_1$ and $T_2$. An \emph{alignment} $\mathcal{A}$ between $T_1$ and $T_2$ is a pair of injective mappings $\mu : I'_1 \to I'_2,\ \nu : J'_1 \to J'_2$, where $I'_1 \subseteq I_1,\ I'_2 \subseteq I_2$ and $J'_1 \subseteq J_1,\ J'_2 \subseteq J_2$ satisfy $|I'_1|=|I'_2|$ and $|J'_1|=|J'_2|$. Intuitively, $\mu$ selects and reorders a subset of rows of $T_1$ to match a subset of rows of $T_2$, and $\nu$ does the same for columns respectively. 
Given $\mathcal{A} = (\mu, \nu)$, the \emph{overlap} induced by $\mathcal{A}$ is the set of matched cells with \emph{identical values}, i.e., the same string, number, or null: $O_{\mathcal{A}} = \big\{ (i,\,j) \in I'_1 \!\times\! J'_1 \;\big|\; c^{(1)}_{i,\,j} = c^{(2)}_{\mu(i),\,\nu(j)} \big\}$. The cardinality $|O_{\mathcal{A}}|$ is the overlap size under $\mathcal{A}$.

\revfour{Our overlap definition is based on exact cell equality. Thus, a pair of null values is counted as a match, just as two identical strings or numbers are. In some applications, however, overlap may be defined only over observed values, in which case pairs of null values should be excluded. More refined definitions are possible when nulls have different meanings. For example, a null may denote an inapplicable attribute, where matching nulls is reasonable. In other cases, it may indicate the value is unknown, where matching nulls may be misleading. Distinguishing these cases requires dataset-specific null semantics and labeling rules, which are beyond the scope of this paper.}

\smallskip\noindent\textbf{Largest Overlap.}
Let $\mathcal{A}^*$ be \revfour{an optimal alignment} that maximizes the number of matched cells, i.e., $\mathcal{A}^* \in \arg\max_{\mathcal{A}} |O_{\mathcal{A}}|$. We define the \emph{largest overlap} as the corresponding overlap $O^* = O_{\mathcal{A}^*}$. Intuitively, $O^*$ is \revfour{a maximum-size} common subtable obtainable between $T_1$ and $T_2$ after optimally reordering their rows and columns.

\smallskip\noindent\textbf{Overlap Ratio.}
To compare tables of differing sizes, we normalize the size of the largest overlap by the area of the smaller table~\cite{pugnaloni2025table}. The resulting \emph{overlap ratio} $\theta$ is defined as $\theta(T_1,\ T_2) = \frac{|O^*|}{\min(|T_1|,\,|T_2|)},\ \theta \in [0,1]$.

\smallskip\noindent\textbf{Problem Definition.}
We now formalize overlap ratio estimation as a regression problem.
\begin{definition}[Overlap Ratio Estimation]
    Given two tables $T_1$ and $T_2$, let $\theta$ be their ground-truth overlap ratio. The task is to learn a parameterized model $M_{\phi}$ that predicts $\hat{\theta}=M_{\phi}(T_1,T_2)$ to approximate $\theta$.
    Given a training set of labeled pairs $\mathcal{D}=\{(T_1, T_2, \theta)\}$, \revfour{the learning objective is to obtain parameters $\phi^*$} that minimize the prediction error, i.e.,   $\phi^*=\arg\min_{\phi}\sum_{(T_1,T_2,\theta)\in\mathcal{D}}\ell\!\left(M_{\phi}(T_1,T_2),\,\theta\right)$,
    where $\ell(\cdot,\cdot)$ measures the error between $\hat{\theta}$ and $\theta$.
\end{definition}

\subsection{State of the Art}

The state of the art in learning-based overlap ratio estimation is \Arma{}~\cite{pugnaloni2025table}. It learns table embeddings whose cosine similarity approximates the overlap ratio, thereby addressing the scalability limitations of exact solver. \Arma{} follows a design paradigm that we characterize as the \textbf{\ec{}} 
paradigm:

\noindent\textbf{(1) \encode{}.} Each table is independently mapped to a tripartite graph with row, column, and cell nodes. Undirected edges connect cell nodes to their row and column nodes. A graph neural network (GNN), GraphSAGE~\cite{hamilton2017inductive}, is employed to refine the node embeddings and a mean pooling readout produces a single embedding vector as the final table representation.

\noindent\textbf{(2) \compare{}.} Given two tables, the model approximates their overlap ratio by computing the cosine similarity between their respective embeddings. This similarity score is treated as the predicted overlap ratio $\hat{\theta}$.

This architecture allows for efficient estimation, as embedding a table requires only linear-time graph construction, and the final estimation reduces to a simple vector similarity computation. 

\section{Design Rationale for Overlap Ratio}

\begin{figure*}[t]
    \centering
    \includegraphics[width=0.99\textwidth]{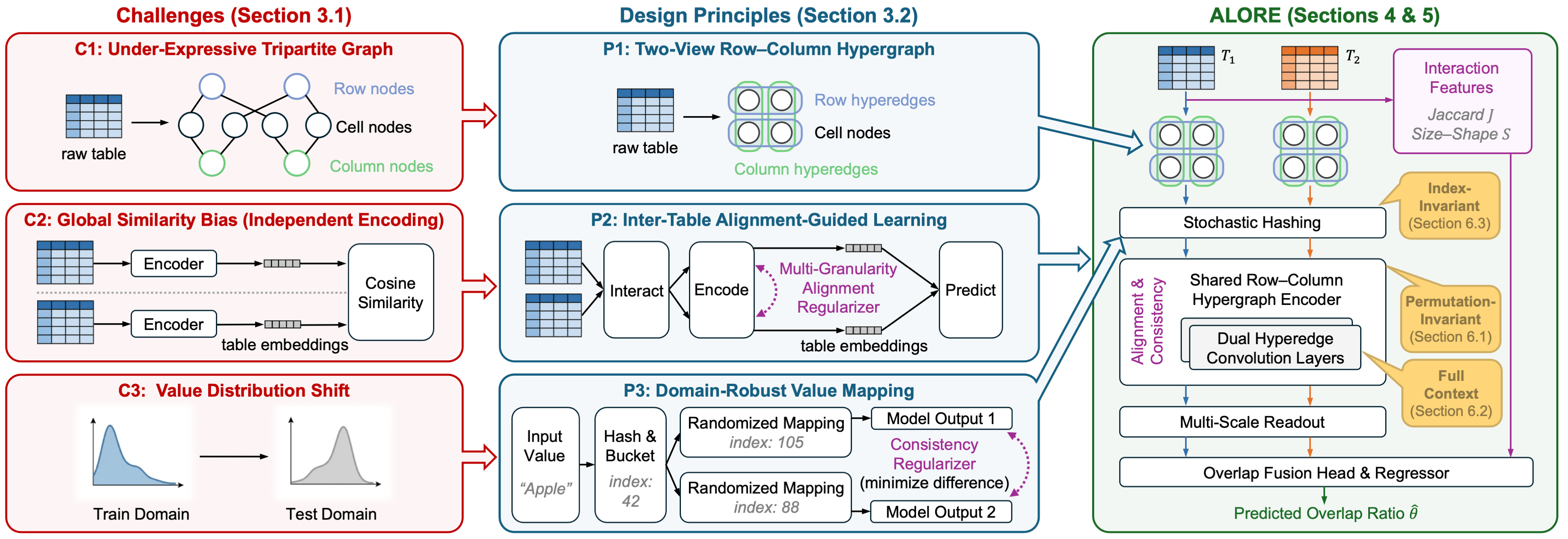}
    \caption[]{Overview of \Our{}: challenges in overlap estimation (left), design principles (middle), \Our{} architecture (right).}
    \label{fig:method}
\end{figure*}

The overlap ratio is defined by the optimal joint alignment of rows and columns across two tables. This makes overlap ratio estimation fundamentally different from global table similarity. Under this definition, two tables share content only when subsets of rows and columns can be jointly aligned into a common subtable. However, prevailing \ec{} estimator embeds each table independently and compares the resulting embeddings, which biases learning toward global similarity rather than alignment evidence. Figure~\ref{fig:method} provides the roadmap for this section: the left panel summarizes the three challenges that follow from this mismatch (Section~\ref{subsec:challenges}), the middle panel distills the corresponding design principles (Section~\ref{subsec:principles}), and the right panel shows how \Our{}, our overlap ratio estimator, realizes them end-to-end (Section~\ref{subsec:alore_overview}).

\subsection{Challenges}\label{subsec:challenges}

As illustrated in Figure~\ref{fig:method} (left), the \ec{} paradigm induces a bias toward global similarity, whereas overlap is driven by whether subsets of rows and columns can be aligned into correspondence. In particular, its tripartite graph representation and independent encoding make \revfour{row--column structure} and pair-specific correspondences only indirectly observable. This gap is amplified in cross-domain settings, where value distributions shift substantially. We next characterize three concrete challenges that follow.

\smallskip\noindent\textbf{Challenge 1:~Tripartite graphs underexpress joint row--column structure.} 
\Arma{} represents a table as a tripartite graph. While this faithfully captures tabular adjacency and separate row and column membership signals, it does not make joint \revfour{row--column structure} explicit. A cell contributes to the largest overlap only when its row and column can both be aligned across two tables. In the tripartite graph, this joint row--column dependency must be inferred indirectly through multi-hop message passing. 

This indirect inference creates a learning challenge. Each row and column node aggregates information from all of its incident cells. As tables grow, many distinct row--column contexts are blended into shared row and column embeddings. It becomes increasingly difficult for the encoder to preserve which specific row–column combinations consistently co-occur with the same values and could jointly participate in an overlapping region. The resulting representation captures how rows and columns behave individually, but it is not designed to preserve the higher-order row–column structure that determines which subsets of rows and columns can be permuted into a large aligned subtable.

\smallskip\noindent\textbf{Challenge 2:~Independent encoding biases the model toward global similarity.}
\Arma{} embeds each table independently and derives overlap ratio through a similarity function. As a result, the model never sees the two tables jointly, and thus cannot adapt its representations to the specific alignment between them. This has two consequences: (i)~\emph{No explicit value correspondences}. The model does not observe which value occurrences across the two tables tend to participate in optimal permutations; (ii)~\emph{Limited structural compatibility signals}. The model cannot directly assess whether subsets of rows and columns in one table are compatible with subsets in the other, because each encoding is computed in isolation.

True overlap is defined by a \emph{maximal aligned structure}. This depends on localized correspondences (e.g., a subset of matching rows/columns), not just global similarity. Under independent encoding with global pooling, such fine-grained evidence is diluted. For example, a single highly overlapping column may dominate the true overlap, yet its signal is averaged with many unrelated cells when forming a single table embedding. Conversely, two tables may appear globally similar in value statistics but have incompatible row–-column arrangements, leading to small true overlap despite high embedding similarity.
This behavior mirrors a well-known limitation of bi-encoders in retrieval and matching tasks~\cite{reimers2019sentence}, in contrast to cross-encoders that perform interaction across pairs~\cite{vaswani2017attention, khattab2020colbert}. \Arma{} therefore tends to emphasize global content similarity, whereas overlap ratio estimation requires alignment-aware evidence about which parts of the two tables can be put into correspondence.

\smallskip\noindent\textbf{Challenge 3:~Value distribution shift causes cross-domain degradation.} 
\Arma{} initializes cell features using SHA-256 hashing of raw cell values, and then refines them via GraphSAGE layers. In principle, hashing is domain-agnostic. Equal values map to equal hashes, and the semantics of the value do not matter for overlap. In practice, the learned embeddings capture the distribution of the training corpus.

Different corpora often exhibit distinct distributions (e.g., tables dominated by strings or floats). When hashed values are propagated and aggregated via global pooling, the encoder inevitably learns these corpus-specific distributions. As a result, the learned embedding space becomes specific to the training domain rather than capturing domain-invariant signals needed for alignment-based overlap ratio estimation.
Empirically, this manifests as substantial cross-domain degradation. When trained on one corpus and evaluated on the other, \Arma{}'s MAE triples, and in some settings it is outperformed by simple Jaccard-based baselines that ignore tabular structure. This suggests that the model conflates overlap with corpus-specific value distributions instead of focusing on domain-invariant structural and relational cues.

\subsection{Design Principles}\label{subsec:principles}
Figure~\ref{fig:method} (middle) summarizes three principles that address Challenges~1--3, which we detail next.

\smallskip\noindent\textbf{Principle 1:~Structure-Expressive Encoding via a Two-View Hypergraph.}
Challenge~1 shows that the tripartite graph exposes row and column information only as separate, pairwise signals, leaving joint row–column structure to be reconstructed indirectly. We therefore introduce a Two-View Row--Column Hypergraph, in which each row and each column is modeled as a hyperedge over cell nodes. Row and column hyperedges form two complementary ``views'' of the same set of cells. This design follows the broader insight that hypergraph neural networks more naturally capture higher-order group structure than pairwise GNNs~\cite{chen2023hytrel, feng2019hypergraph, yadati2019hypergcn}, but we specialize it to overlap ratio estimation. Our two-view row--column hypergraph preserves the multi-cell behavior of rows and columns that is stable under permutations, so that the encoder can more easily identify combinations of rows and columns that can support a joint alignment between tables.

\smallskip\noindent\textbf{Principle 2:~Inter-Table Alignment-Guided Learning.}
Challenge~2 stems from the \ec{} paradigm, since encoding each table independently and comparing only at the end deprives the encoder of any inter-table alignment signals. Motivated by this, we move from \ec{} to an \textbf{\iep{}} paradigm and make the encoder explicitly \emph{alignment-guided}. We retain an efficient per-table encoder (the Row--Column Hypergraph Encoder), but train it on table pairs with an overlap regression loss and a multi-granularity alignment regularizer. These pairwise objectives evaluate how well rows, columns, and shared values across the two tables can be aligned and send gradients back through the encoder. At prediction time, an overlap fusion head combines the two table embeddings with a small set of inter-table signals before regressing the overlap ratio. In combination, this introduces lightweight inter-table interaction while preserving bi-encoder efficiency, directly addressing the structural limitation of the original \ec{}'s design.

\smallskip\noindent\textbf{Principle 3:~Domain-Robust Value Mapping.}
\revfour{Challenge~3 shows that value representations should preserve exact equality without overfitting to corpus-specific value distributions. Inspired by domain randomization and invariant learning~\cite{arjovsky2019invariant, tobin2017domain, peng2018sim}, we treat cell values as symbolic identities: equal values must remain equal, since overlap is defined by exact cell value matches, but the model should not learn that particular values are characteristic of a training corpus. We therefore hash values into embedding indices and randomly relabel these indices during training. Each relabeling preserves equality within a table pair, so exact matches remain observable, while preventing any index from carrying domain-specific information. A consistency regularizer further encourages predictions to remain stable under such relabelings. Together, these mechanisms discourage memorization of training-domain value distributions and bias the model toward row/column co-occurrence and alignment patterns that are more stable under domain shift.}

\subsection{Overview of \Our{}}\label{subsec:alore_overview}

Figure~\ref{fig:method} (right) instantiates Principles~1--3 in \Our{}. In contrast to the \ec{} paradigm designed for global similarity, we organize \Our{} as an \iep{} pipeline tailored to overlap ratio estimation. This clarifies where inter-table interaction enters: through inexpensive interaction features at input time and through alignment-guided objectives during training, while inference remains a bi-encoder-style forward pass.

\noindent\textbf{(1) \interact{}.} We introduce inter-table interaction through lightweight features computed from the raw table pair, including a Jaccard-based value overlap score and size--shape descriptors derived from table dimensions. These features provide heuristic guidance for predicting the overlap ratio, while the alignment-guided objectives in the \predict{} stage provide complementary interaction during training, without computing an explicit alignment.

\noindent\textbf{(2) \encode{}.} Each table is transformed to a Two-View Row--Column Hypergraph. A shared Row--Column Hypergraph Encoder then processes this structure using stacked dual hyperedge convolution layers and a multi-scale table readout, producing a table embedding from each encoder layer. A domain-robust value mapping initializes cell embeddings using hash-based indices and stochastic permutations during training, encouraging the encoder to focus on structural compatibility rather than corpus-specific value identities.

\noindent\textbf{(3) \predict{}.} An overlap fusion head combines the two multi-scale table embeddings with the interaction features. Finally, an overlap ratio regressor outputs the final estimated overlap ratio. Training uses an overlap regression loss as the main objective, a multi-granularity alignment regularizer that sends pairwise alignment signals back into the encoder, and a consistency regularizer that stabilizes predictions under randomized value mappings.

We instantiate these principles by formalizing hypergraph encoding with domain-robust value mapping (Section~\ref{sec:mtd-encoding}) and interaction-aware prediction with alignment-guided objectives (Section~\ref{sec:mtd-training}).

\section{Domain-Robust Hypergraph Encoding}\label{sec:mtd-encoding}

\renewcommand{\algorithmcfname}{Procedure}
\begin{algorithm}[t]
    \caption{Training Step of \Our{}}\label{algo:train}
    \SetKwInOut{Input}{Input}
    \SetKwInOut{Output}{Output}
    \SetKwFunction{FTrain}{\textbf{TrainStep}}
    \SetKwFunction{FFwd}{\textbf{Forward}}
    \SetKwProg{Fn}{}{:}{}
    \small
    \Input{A pair of tables $(T_1,T_2)$, a ground-truth overlap ratio $\theta$, the number of hash buckets $B$, a gating threshold and exponent $(\kappa,\gamma)$, and a set of weights $(\lambda_{\mathrm{RC}},\lambda_{\mathrm{val}},\lambda_{\mathrm{ctx}},\lambda_{\mathrm{cons}})$}
    \Output{Training loss $\mathcal{L}_{\mathrm{total}}$}

    \smallskip
    \revthree{
    \Fn{\FFwd{$T_1,T_2,I,\pi$}}{
        $(H_1,Z_1) \leftarrow \text{Encode}(T_1,\pi)$; \commt{with stochastic value mapping}\\ \nllabel{ln:enc1}
        $(H_2,Z_2) \leftarrow \text{Encode}(T_2,\pi)$\; \nllabel{ln:enc2}
        $\hat{\theta} \leftarrow \text{Predict}(|Z_1-Z_2|, I)$\; \nllabel{ln:pred}
        \Return{$(\hat{\theta},H_1,H_2,Z_1,Z_2)$}\;
    }
    }
    \smallskip
    \Fn{\FTrain{$T_1,T_2,\theta$}}{
        $\pi_1 \leftarrow \text{UniformRandomPerm}(\{0,\dots,B{-}1\})$\; \nllabel{ln:perm1}
        $\pi_2 \leftarrow \text{UniformRandomPerm}(\{0,\dots,B{-}1\})$\; \nllabel{ln:perm2}
        
        $I \leftarrow [\,S(T_1,T_2)\parallel J(T_1,T_2)\,]$; \commt{interaction features}\\ \nllabel{ln:interact}
        
        $(\hat{\theta}_1,H_1,H_2,Z_1,Z_2) \leftarrow \FFwd(T_1,T_2,I,\pi_1)$\; \nllabel{ln:fwd1}
        $(\hat{\theta}_2,\_,\_,\_,\_) \leftarrow \FFwd(T_1,T_2,I,\pi_2)$; \commt{for consistency}\\ \nllabel{ln:fwd2}

        $\mathcal{L}_{\mathrm{main}} \leftarrow |\hat{\theta}_1-\theta|$\; \nllabel{ln:main}
        $W_{\mathrm{gate}} \leftarrow \max\!\left(0,\frac{\theta-\kappa}{1-\kappa}\right)^{\gamma}$\; \nllabel{ln:gate}

        $\mathcal{L}_{\mathrm{align}} \leftarrow 0$\;
        \If{$W_{\mathrm{gate}} > 0$}{
            $\mathcal{L}_{\mathrm{RC}} \leftarrow \text{RowColumnAlign}(H_1,H_2)$\; \nllabel{ln:rc}
            $\mathcal{L}_{\mathrm{val}} \leftarrow \text{ValueAlign}(H_1,H_2)$\; \nllabel{ln:val}
            $\mathcal{L}_{\mathrm{ctx}} \leftarrow \text{ContextAlign}(H_1,H_2,Z_1,Z_2)$\; \nllabel{ln:ctx}
            $\mathcal{L}_{\mathrm{align}} \leftarrow W_{\mathrm{gate}}\!\left(\lambda_{\mathrm{RC}}\mathcal{L}_{\mathrm{RC}}+\lambda_{\mathrm{val}}\mathcal{L}_{\mathrm{val}}+\lambda_{\mathrm{ctx}}\mathcal{L}_{\mathrm{ctx}}\right)$\; \nllabel{ln:align}
        }

        $\mathcal{L}_{\mathrm{cons}} \leftarrow \|\hat{\theta}_1-\hat{\theta}_2\|_2^2$\; \nllabel{ln:cons}
        $\mathcal{L}_{\mathrm{total}} \leftarrow \mathcal{L}_{\mathrm{main}} + \mathcal{L}_{\mathrm{align}} + \lambda_{\mathrm{cons}}\mathcal{L}_{\mathrm{cons}}$\; \nllabel{ln:total}
        Update parameters using $\nabla \mathcal{L}_{\mathrm{total}}$\;
        \Return{$\mathcal{L}_{\mathrm{total}}$}\; \nllabel{ln:update}
    }
\end{algorithm}

\revthree{Figure~\ref{fig:method} (right) summarizes \Our{}'s architecture. Procedure~\ref{algo:train} presents the end-to-end training step as a roadmap for Sections~\ref{sec:mtd-encoding} and~\ref{sec:mtd-training}. This section explains the \encode{} stage through its three main components: 1) the Two-View Row--Column Hypergraph that represents each table; 2) the shared Row--Column Hypergraph Encoder that produces structure-aware embeddings; 3) the domain-robust value mapping used to initialize cell features and improve generalization across corpora. Section~\ref{sec:mtd-training} then presents the \interact{} and \predict{} stages, including the pairwise interaction features, prediction head, and alignment-guided training objectives.}

\subsection{Two-View Row--Column Hypergraph}\label{subsec:mtd-hypergraph}

A central challenge in overlap ratio estimation is to capture how subsets of rows and columns can be permuted into an aligned subtable. To expose this higher-order structure explicitly, we model each table $T$ with \revthree{$m$ rows and $n$ columns} as a Two-View Row--Column Hypergraph $\mathcal{H}(T) = (\mathcal{V}, \mathcal{E}_{\mathrm{row}} \cup \mathcal{E}_{\mathrm{col}})$, where nodes correspond to individual cells and hyperedges correspond to entire rows and columns, \revthree{as part of the encode step in Procedure~\ref{algo:train} (lines~\ref{ln:enc1}--\ref{ln:enc2}).}

\smallskip\noindent\textbf{Nodes and Features.}
Each cell $c_{i,\,j}$ is mapped to a node $v_{ij} \in \mathcal{V}$. We apply SHA-256 to the raw cell value $c_{i,\,j}$ to initialize node feature. This feature is used only as an identity surrogate. Unlike semantic word embeddings~\cite{devlin2019bert, liu2019roberta}, this hashing scheme is content-agnostic and preserves exact equality, matching the definition of overlap.

\smallskip\noindent\textbf{Row/Column Hyperedges.}
To expose the inherent two-dimensional groupings of rows and columns, we introduce two families of hyperedges: the \emph{row hyperedges} $\mathcal{E}_{\mathrm{row}}=\{e^r_i \mid e^r_i=\{v_{i,1},\dots,v_{i,n}\}\}$ and the \emph{column hyperedges} $\mathcal{E}_{\mathrm{col}}=\{e^c_j \mid e^c_j=\{v_{1,j},\dots,v_{m,j}\}\}$. Each row hyperedge $e^r_i$ connects all cells in row $i$, and each column hyperedge $e^c_j$ connects all cells in column $j$. As a result, every cell node participates in exactly two hyperedges, one row and one column, which exposes joint row--column structure explicitly.

This construction gives the encoder direct access to full row- and column-level context within a \emph{single} hypergraph convolution, rather than reconstructing it from multiple hops over cell--row and cell--column binary edges. It also matches the inductive bias of overlap. Permuting rows or columns does not change $\mathcal{E}_{\mathrm{row}}$ or $\mathcal{E}_{\mathrm{col}}$, so the hypergraph topology is preserved under any reordering consistent with the overlap definition. Combined with symmetric hypergraph convolutions and permutation-invariant pooling in the Row--Column Hypergraph Encoder (Section~\ref{subsec:mtd-encoder}), the resulting table embedding is invariant to row and column permutations.

Finally, the dual membership of each cell enables efficient propagation along both axes. A single Dual Hyperedge Convolution layer can pass information across all rows and columns, enabling the encoder to capture row--column combinations associated with consistent value patterns. In contrast, exposing the same signal with a tripartite graph typically requires several hops of pairwise message passing and is more susceptible to oversmoothing~\cite{xu2019how, li2018deeper}.

\smallskip\noindent\uline{Complexity}.
Constructing $\mathcal{H}(T)$ requires a single pass over all cells to create nodes, compute hashes, and assign hyperedge memberships. The resulting time and space complexity is $\mathcal{O}(|T|)$, comparable to building \Arma{}'s tripartite graph while offering a structure more conducive to alignment. The resulting incidence structure can be stored in sparse format for efficient batching.

\subsection{Row–Column Hypergraph Encoder}\label{subsec:mtd-encoder}

\Our{} employs a Row--Column Hypergraph Encoder to learn structure-aware cell embeddings from the two-view hypergraph (lines~\ref{ln:enc1}--\ref{ln:enc2}). The encoder is shared between the two tables, forming a Siamese architecture~\cite{bromley1993signature}. Crucially, this does not require cross-attention between all cell pairs, which would be computationally prohibitive for large tables. Unlike previous approaches that inject interaction signals directly into the encoding process, we maintain a pure structural encoder to learn domain-invariant features.

\smallskip\noindent\textbf{Dual Hyperedge Convolution Layer.}
The encoder stacks $L$ \emph{dual hyperedge convolution layers}. Let $H^{(l)} \in \mathbb{R}^{|\mathcal{V}| \times d}$ denote the node \ embeddings after layer $l$, where $d$ is the embedding dimension. Given $H^{(l-1)}$, the $l$-th layer applies two independent hypergraph convolutions~\cite{gao2023hgnn} to obtain $H^{(l)}_{\mathrm{row}} = \mathrm{HConv}(H^{(l-1)}, \mathcal{E}_{\mathrm{row}})$ and $H^{(l)}_{\mathrm{col}} = \mathrm{HConv}(H^{(l-1)}, \mathcal{E}_{\mathrm{col}})$. The two views are then fused:
$H^{(l)}=\penalty-10000
\operatorname{ReLU}\bigl(
\operatorname{LayerNorm}\bigl(
\operatorname{Linear}\bigl(
[H^{(l)}_{\mathrm{row}}\mathbin{\parallel}H^{(l)}_{\mathrm{col}}]
\bigr)\bigr)\bigr)$,
where $[\cdot \parallel \cdot]$ denotes concatenation. This design ensures that each cell update is informed by both its row and column context. The layer is permutation-invariant to row/column orderings and runs in time linear in the number of cell--hyperedge incidences.

\smallskip\noindent\textbf{Multi-Scale Table Readout.}
To capture structure at different depths, the encoder exposes a \emph{multi-scale table readout}. At each layer $l \in {1,\dots,L}$, an inner projection first applies a lightweight MLP to the node embeddings $H^{(l)}$. We then aggregate the projected nodes with permutation-invariant global sum pooling to obtain a layer-wise table vector $z^{(l)}_{\mathrm{pool}}$. An outer projection finally refines $z^{(l)}_{\mathrm{pool}}$ via another MLP, yielding the readout $z^{(l)}$. We use sum pooling since it preserves multiplicities and yields a more expressive multiset readout than averaging~\cite{xu2019how}. The final table representation concatenates all layer-wise vectors and normalizes them, i.e., $Z = \text{LayerNorm}([z^{(1)} \parallel \cdots \parallel z^{(L)}])$. This multi-scale representation serves two roles. First, the final-layer embedding dominates the main prediction task, giving the overlap fusion head access to a rich summary. Second, the intermediate embeddings $H^{(1)},\dots,H^{(L)}$ together with readouts $z^{(l)}$ parameterize our multi-granularity alignment regularizer (Section~\ref{subsec:mtd-obj}), which enforces alignment at row, column, and value levels.

\smallskip\noindent\uline{Complexity}.
Each dual hyperedge convolution layer processes every cell once in its row hyperedge and once in its column hyperedge, so the cost per layer is $\mathcal{O}(|T| d)$ for embedding dimension $d$. With $L$ layers, the Row--Column Hypergraph Encoder runs in $\mathcal{O}(L |T| d)$ time and $\mathcal{O}(|T| d)$ space, matching the asymptotic order of GNNs on \Arma{}’s tripartite graph.

\subsection{Domain-Robust Value Mapping}\label{subsec:mtd-valuemap}

Even with symbolic hashing, a fixed embedding table can leak corpus-specific frequency patterns into the learned space. This hurts transfer because the estimator may partially rely on which hashed buckets are common in the training corpus rather than on alignment-relevant structure. Principle~3 addresses this by randomizing the association between hashed values and embedding indices during training, while preserving exact value equality within each forward pass.

\smallskip\noindent\textbf{Bucketization and Stochastic Permutation.}
\revthree{Let $B$ denote the number of hash buckets, which fixes the size of the value-embedding table. Each hashed feature $x_{i, j}$ is mapped into a fixed index range $[0, B-1]$ via a modulo operation $\mathrm{idx}_{i, j} = x_{i, j} \bmod B$, which defines a bounded embedding table. In \texttt{Forward}, the encoder uses a supplied permutation $\pi$ and performs embedding lookup with the permuted indices $\pi(\mathrm{idx}_{i,j})$ during encoding (lines~\ref{ln:enc1}--\ref{ln:enc2}). During training, \texttt{TrainStep} draws two independent random permutations over $\{0,\dots,B-1\}$ for the two forward passes (lines~\ref{ln:perm1}--\ref{ln:perm2}).}
This mechanism preserves the equality relation within a training step: \emph{two cells share an embedding if and only if they share the same hashed value and thus the same permuted index.} At the same time, it breaks any stable association between a specific index and its global distribution across training epochs. The encoder cannot memorize that a particular index corresponds to a corpus-specific keyword. Therefore, it must instead rely on structural patterns, such as how the same value appears in compatible rows and columns across the two tables.

At inference time, the permutation is disabled and embeddings are looked up directly using $\mathrm{idx}_{i, j}$. The mapping remains deterministic and efficient while the model has already learned to be insensitive to the absolute positions of embedding indices.

\smallskip\noindent\textbf{Size--Shape Normalization Features.}
Value embeddings alone do not adjust for distribution shifts in table sizes and aspect ratios. We therefore complement the node-level mapping with \emph{size–shape features} that capture relative shape and scale rather than absolute magnitudes. For a table pair $(T_1,\,T_2)$, we compute a 7-dimensional vector $S(T_1,\,T_2)$ from the log-transformed row and column counts of both tables $\log m_1$, $\log n_1$, $\log m_2$, $\log n_2$, the log of the smaller table area $\log (\min(|T_1|,\,|T_2|))$, and the ratios between row and column counts $\frac{\min(m_1, m_2)}{\max(m_1, m_2)}$, $\frac{\min(n_1, n_2)}{\max(n_1, n_2)}$. These features are then fed into the overlap fusion head after a scaling and mixing step. They stabilize the representation across corpora with different table size distributions.

Together, the stochastic value mapping and size--shape features decouple identity from distributional context. The permutation mechanism forces the encoder to treat value identities as purely symbolic rather than tying them to persistent embeddings associated with corpus-specific distributions. The size--shape features, in turn, factor out absolute table magnitudes and anchor the model on relative shape- and size-related signals that transfer across corpora. As a result, the learned representation retains equality information needed for overlap while reducing sensitivity to domain-specific value or size distributions. Finally, during training we couple this randomized mapping with a consistency regularizer (Section~\ref{subsec:mtd-obj}) that enforces stable predictions across independently sampled permutations (lines~\ref{ln:fwd1}--\ref{ln:fwd2},~\ref{ln:cons}). This encourages invariance to bucket-index relabeling rather than memorization of bucket identities.

\smallskip\noindent\uline{Complexity}.
Bucketization and stochastic permutation operate in $\mathcal{O}(|T|)$ time per table, plus $\mathcal{O}(B)$ time per batch to generate a new permutation. Computing $S(T_1,\,T_2)$ requires only simple scalar operations on row and column counts and is negligible compared to hypergraph encoding.

\section{Alignment-Guided Prediction}\label{sec:mtd-training}

Section~\ref{sec:mtd-encoding} describes how \Our{} encodes each table into a structure-aware and domain-robust representation. We now turn to the pairwise components that make these representations predictive of alignment-dependent overlap. As shown in Figure~\ref{fig:method} (right), \Our{} (i)~augments the encoded table representations with inexpensive interaction features computed directly from the raw pair, and (ii)~trains the shared encoder with alignment-guided objectives that expose correspondence signals without requiring explicit row--column alignments. We first describe the interaction-aware regressor, and then present the training losses and regularizers, \revthree{following the prediction and loss computation steps in Procedure~\ref{algo:train}.}

\subsection{Interaction-Aware Regressor}\label{subsec:mtd-pred}

The \predict{} stage combines (i)~structure-aware embeddings produced by the shared encoder and (ii)~lightweight interaction features computed directly from the raw pair. This design injects cheap inter-table signals without quadratic cross-attention over tables.

\smallskip\noindent\textbf{Inter-Table Interaction Features.}
For a table pair $(T_1,\,T_2)$, we construct a compact \emph{interaction feature vector} $I(T_1,\,T_2) = \big[S(T_1,\,T_2)\mathbin{\Vert} J(T_1,\,T_2)\big]$. Here, $S(T_1,\,T_2)$ is the 7-dimensional \emph{size--shape descriptor} defined earlier (from log-transformed table dimensions and row/column ratios), and $J(T_1,\,T_2)$ is a \emph{Jaccard anchor} defined as the Jaccard similarity between the sets of cell values of $T_1$ and $T_2$. These interaction features do not attempt to reconstruct the optimal permutation. Instead, they serve as signals about value co-occurrence and relative table shape and size. These help the predictor to estimate plausible overlap ratios without relying solely on the learned structural embeddings (line~\ref{ln:interact}).

\smallskip\noindent\textbf{Overlap Fusion Head and Overlap Ratio Regressor.}
Let $Z_1$ and $Z_2$ be the multi-scale table embeddings produced by the Row–Column Hypergraph Encoder and multi-scale table readout. We first form a structural difference feature $Z_{\text{diff}} = \big|Z_1 - Z_2\big|$. We pass the interaction vector $I(T_1,\,T_2)$ through a small interaction feature encoder (a linear projection with normalization and non-linearity) to obtain a hidden vector $I_{\text{hid}}$ (line~\ref{ln:pred}).

The \emph{overlap fusion head} concatenates structural and interaction representations $\big[Z_{\text{diff}} \parallel I_{\text{hid}}\big]$ and feeds them into the \emph{overlap ratio regressor}, a two-layer MLP that outputs the predicted overlap ratio $\hat{\theta} \in [0,1]$. This allows the predictor to combine fine-grained structural similarity with inter-table interaction signals (line~\ref{ln:pred}).

\smallskip\noindent\uline{Complexity}.
Computing $J(T_1,T_2)$ is near-linear in $|T_1| + |T_2|$ using batched hash-based set operations, whereas computing $S(T_1,T_2)$ is constant time. The interaction feature encoder and Overlap Fusion Head operate on low-dimensional vectors and have cost $\mathcal{O}(d)$ per pair, independent of table size.

\subsection{Alignment-Guided Objectives}\label{subsec:mtd-obj}

\Our{} is trained with a primary overlap regression loss and two regularizers. The multi-granularity alignment regularizer injects row-, column-, and value-level correspondence signals into the shared encoder without requiring discrete alignment supervision. The consistency regularizer enforces prediction stability under the stochastic value permutations of domain-robust value mapping. Both regularizers are used only during training, so inference remains a single bi-encoder forward pass.

\smallskip\noindent\textbf{Overlap Regression Loss.}
The main objective is the \emph{overlap regression loss}, an L1 loss between the predicted overlap ratio and the ground truth (line~\ref{ln:main}): $\mathcal{L}_{\text{main}} = \left|\, \hat{\theta}(T_1, T_2) - \theta(T_1, T_2) \,\right|$.

\smallskip\noindent\textbf{Multi-Granularity Alignment Regularizer.}
The \emph{multi-granularity alignment regularizer} encourages the encoder to produce embeddings that admit the inter-table alignments needed for overlap, without requiring explicit row or column match labels. It provides complementary alignment signals at the row, column, and value levels, including constraints that account for how shared values participate in each table’s row--column structure. The regularizer comprises three components and combines them with a gating factor derived from the ground-truth overlap to avoid enforcing alignments on pairs with negligible overlap.

\smallskip\noindent\textbf{(i) RowColumn-Align.} 
\emph{RowColumn-Align} first summarizes cell-level embeddings into row proxies and column proxies by mean pooling over the corresponding row and column hyperedges. It then encourages the two tables to admit a consistent row and column correspondence using entropic optimal transport (OT)~\cite{cuturi2013sinkhorn}. Concretely, for the row proxies of $T_1$ and $T_2$, we compute a soft matching matrix $P$ with non-negative entries. The matrix is normalized so that each row proxy distributes a fixed total amount of matching weight across the other table, yielding a stable ``soft assignment''. A Sinkhorn-like normalization computes $P$ efficiently~\cite{cuturi2013sinkhorn}. Using $P$, we reconstruct each row proxy of $T_1$ as a weighted combination of row proxies from $T_2$, and vice versa. The row alignment loss is the mean-squared reconstruction error from these two reconstructions. We apply the same procedure to the column proxies to obtain a column alignment loss, and combine the two as $\mathcal{L}_{\text{RC}} = \mathcal{L}_{\text{row}} + \mathcal{L}_{\text{col}}$.

From an optimization viewpoint, this OT-based construction replaces discrete row and column permutation search with a continuous soft matching that remains differentiable end-to-end. When the two proxy sets have equal size and uniform normalization, the feasible soft matchings include the classical doubly-stochastic relaxations of permutation matrices (a standard convex relaxation of assignments)~\cite{mena2018learning}. In contrast, \Sloth{} optimizes directly over discrete permutations. RowColumn-Align therefore biases the encoder toward representations for which a low-cost row and column correspondence exists, while avoiding combinatorial optimization.

\smallskip\noindent\textbf{(ii) Value-Align.}
\emph{Value-Align} enforces identity consistency for values shared across the two tables. Let $\mathcal{V}_{\cap}$ denote the set of hashed values that appear in both $T_1$ and $T_2$. For each $v \in \mathcal{V}_{\cap}$, we construct a value proxy $h_v^{(k)}$ by mean pooling the embeddings of all cell nodes in $T_k$ whose value equals $v$. We then penalize the squared Euclidean discrepancy between the two proxies: $\mathcal{L}_{\text{val}} = \sum_{v \in \mathcal{V}_{\cap}} \| h_v^{(1)} - h_v^{(2)} \|_2^2$. This term encourages identical values to be encoded compatibly across tables despite differences in local neighborhoods.

\smallskip\noindent\textbf{(iii) Context-Align.}
Value identity alone can be misleading when the same cell value appears in both tables but serves different roles. \emph{Context-Align} therefore enforces role consistency by comparing each shared value’s relationship to the global table embedding.

Let $Z_1$ and $Z_2$ be the global embeddings of $T_1$ and $T_2$. For each $v \in \mathcal{V}_{\cap}$, we compare the distance from the value proxy to the corresponding table embedding and penalize discrepancies: $\mathcal{L}_{\text{ctx}} = \sum_{v \in \mathcal{V}_{\cap}} \Big| \, \| h_v^{(1)} - Z_1 \|_2 - \| h_v^{(2)} - Z_2 \|_2 \, \Big|$. Intuitively, truly overlapping cell values tend to occur in comparable contexts, so their distances to the corresponding global embeddings should be similar. When a shared value appears in mismatched contexts, these distances diverge. Thus, Context-Align penalizes the mismatch and discourages the model from over-relying on coincidental value matches that lack structural correspondence.

\smallskip Let $\mathcal{L}_{\text{RC}}$, $\mathcal{L}_{\text{val}}$, and $\mathcal{L}_{\text{ctx}}$ denote the three terms. Their combination is gated by the ground-truth overlap ratio using an \emph{alignment gating threshold} $\kappa \in [0,1)$: $W_{\text{gate}} = \max\left(0, \frac{\theta - \kappa}{1 - \kappa}\right)^\gamma$, and $\mathcal{L}_{\text{align}} = W_{\text{gate}} \left(\lambda_{\text{RC}}\mathcal{L}_{\text{RC}} + \lambda_{\text{val}}\mathcal{L}_{\text{val}} + \lambda_{\text{ctx}}\mathcal{L}_{\text{ctx}}\right)$. Pairs with small true overlap (i.e., $\theta < \kappa$) do not contribute to alignment, preventing the model from forcing alignments where none exist (lines~\ref{ln:gate}--\ref{ln:align}).

\smallskip\noindent\textbf{Consistency Regularizer.}
The \emph{consistency regularizer} enforces stability under stochastic value permutations~\cite{xie2020unsupervised, sajjadi2016regularization}. For each batch, the model performs two forward passes with independently sampled permutations in the domain-robust value mapping (lines~\ref{ln:perm1}--\ref{ln:fwd2}). This results in two predictions $\hat{\theta}_1$ and $\hat{\theta}_2$. The regularizer is $\mathcal{L}_{\text{cons}} = \|\hat{\theta}_1 - \hat{\theta}_2\|_2^2$. This term encourages the prediction to focus on value equality, instead of particular embedding index permutation and corpus-specific distributions (line~\ref{ln:cons}).

\smallskip\noindent\textbf{Total Objective.}
The final training loss is a weighted sum $\mathcal{L}_{\text{total}} = \mathcal{L}_{\text{main}} + \mathcal{L}_{\text{align}} + \lambda_{\text{cons}}\mathcal{L}_{\text{cons}}$. The alignment and consistency terms are omitted at inference for efficiency (line~\ref{ln:total}).

\smallskip\noindent\uline{Complexity}.
The overlap regression loss is negligible compared to encoding. The alignment regularizer adds a term that is $\mathcal{O}((m^2 + n^2) d)$ per pair due to soft row and column alignments. However, this is evaluated only during training and on tables with modest sizes. The consistency regularizer requires an additional forward pass per batch during training. At inference time, since both regularizers are disabled, the prediction cost is identical to a single forward pass of the encoder and overlap fusion head. 

\section{Theoretical Analysis}\label{sec:theory}

In this section, we show that \Our{} is aligned with the symmetries and structure of table overlap. We establish (i)~invariance to row/column permutation, (ii)~full row and column context per encoder layer, and (iii)~invariance to bucket-index relabeling. We also summarize end-to-end inference time and space complexity.

\subsection{Row/Column Permutation Invariance}\label{subsec:theo-invariance}

\uline{Key Result}. Our encoder matches the symmetry of the overlap definition. Reordering rows or columns only permutes cell-level representations, while the pooled table embedding and the final prediction remain unchanged.

Let $T$ be a table with $m$ rows and $n$ columns. Let $\sigma \in S_m$ and $\rho \in S_n$ be permutations of rows and columns, inducing a permutation of cell indices $(i,j)\mapsto (\sigma(i),\rho(j))$. Let $P_{\sigma,\rho}$ be the corresponding permutation matrix acting on cell-node embeddings (ordered by any fixed enumeration of the $mn$ cells). Let $\mathcal{H}(T)=(\mathcal{V},\mathcal{E}_{\mathrm{row}}\cup\mathcal{E}_{\mathrm{col}})$ be the Two-View Row--Column Hypergraph, where each cell is a node $v_{i,j}\in\mathcal{V}$, each row is a hyperedge $e^r_i=\{v_{i,1},\dots,v_{i,n}\}$, and each column is a hyperedge $e^c_j=\{v_{1,j},\dots,v_{m,j}\}$.

\smallskip\noindent\emph{Assumption 1~(Permutation-equivariant message passing).}
\revfour{In each Dual Hyperedge Convolution layer, the row-view and column-view hypergraph convolutions} aggregate within row and column hyperedges using a symmetric function (sum) and apply the same node-wise transform to every node. Global pooling is sum over cell nodes.

\begin{theorem}[Encoder Equivariance]\label{theorem:enc_equiv}
    Let $H^{(l)}(T) \in \mathbb{R}^{|\mathcal{V}| \times d}$ be the node embedding matrix after $l$ Dual Hyperedge Convolution layers. Under Assumption~1 and value features that depend only on cell contents, for any $\sigma,\rho$, $H^{(l)}(T^{\sigma,\rho}) = P_{\sigma,\rho}H^{(l)}(T)$ for all $l = \{0,\dots,L\}$.
\end{theorem}

\begin{proof}
    We prove the claim by induction on $l$.

    Base case $l = 0$. Node features are computed from cell values only (hashing and embedding lookup). Permuting rows and columns reorders the cells but does not change their values. Therefore, the initial node feature matrix satisfies $H^{(0)}(T^{\sigma,\rho}) = P_{\sigma,\rho}H^{(0)}(T)$.

    Inductive step. Assume $H^{(l-1)}(T^{\sigma,\rho}) = P_{\sigma,\rho}H^{(l-1)}(T)$. Consider the row-view HConv at layer $l$. Each node update is formed by aggregating messages from the row hyperedge and applying node-wise transforms. Under the permutation $(\sigma,\rho)$, the row hyperedges are relabeled, and the multiset of node embeddings inside each hyperedge is permuted. By Assumption~1, the hyperedge aggregation result is unchanged up to the same node reordering $P_{\sigma,\rho}$. The same holds for the node aggregation over incident hyperedges. All subsequent transforms are applied node-wise and thus commute with $P_{\sigma,\rho}$. Hence, the row-view output satisfies $H^{(l)}_{\mathrm{row}}(T^{\sigma,\rho}) = P_{\sigma,\rho}H^{(l)}_{\mathrm{row}}(T)$.
    The argument for the column-view output $H^{(l)}_{\mathrm{col}}$ is identical. Concatenation $[H^{(l)}_{\mathrm{row}} \parallel H^{(l)}_{\mathrm{col}}]$, followed by Linear, LayerNorm and ReLU, preserves equivariance because these operations are applied per node. Therefore, $H^{(l)}(T^{\sigma,\rho}) = P_{\sigma,\rho}H^{(l)}(T)$. This completes the induction.
\end{proof}

\begin{theorem}[Table Embedding Invariance]\label{theorem:tab_emb_inv}
    Let $Z(T)$ be the table embedding obtained by pooling and concatenating layer-wise readouts, followed by LayerNorm. Under Assumption~1, for any $\sigma,\rho$, $Z(T^{\sigma,\rho}) = Z(T)$.
\end{theorem}

\begin{proof}
    We fix any layer $l$. The pooled readout has the form $z^{(l)} = \mathrm{MLP}_{\mathrm{out}}(\mathrm{Pool}(\mathrm{MLP}_{\mathrm{in}}(H^{(l)})))$, where $\mathrm{Pool}$ is sum or mean over nodes. By Theorem~\ref{theorem:enc_equiv}, $H^{(l)}(T^{\sigma,\rho}) = P_{\sigma,\rho}H^{(l)}(T)$. Node-wise MLPs commute with $P_{\sigma,\rho}$. Symmetric pooling satisfies $\mathrm{Pool}(P_{\sigma,\rho}X) = \mathrm{Pool}(X)$. Hence, $z^{(l)}(T^{\sigma,\rho}) = z^{(l)}(T)$ for every $l$. Concatenation over $l$ and LayerNorm on the resulting vector preserve equality.
\end{proof}

Since our inter-table interaction features $S$ and $J$ are also invariant to row/column reorderings, the full predictor $\hat{\theta}(T_1,T_2)$ is invariant to independent permutations applied to $T_1$ and $T_2$.

\subsection{Full Row/Column Context per Layer}\label{subsec:theo-convolution}

\uline{Key Result}. A single dual row--column hypergraph convolution layer lets each cell aggregate from all cells in both its row and column. This yields full row and column context per layer without multi-hop message passing through explicit row/column nodes.

Consider the row view. Let $|T| = mn$ be the number of cells. Let $H_r \in \{0,1\}^{|T| \times m}$ be the cell-row incidence matrix, where $(H_r)_{v,i} = 1$ iff cell-node $v$ belongs to row $i$. We define $D_v$ and $D_e$ as the diagonal degree matrices of nodes and hyperedges in the incidence bipartite graph. A standard mean aggregation hypergraph propagation can be written as $X' = D_v^{-1}H_r D_e^{-1}H_r^\top X W$, for node embedding matrix $X \in \mathbb{R}^{|T| \times d}$ and a learnable linear map $W$. The column view uses the analogous incidence matrix $H_c$.

\begin{proposition}[Two-Hop Equivalence]\label{prop:twohop_equiv}
    Let $\mathcal{B}_r$ be the bipartite incidence graph between cell nodes and row-hyperedge nodes. The linear operator $H_r D_e^{-1}H_r^\top$ aggregates information along length-2 walks in $\mathcal{B}_r$ (cell $\rightarrow$ row $\rightarrow$ cell). Therefore, one row-HConv updates each cell using information from all cells in its row in a single layer. The same holds for the column view.
\end{proposition}

\begin{proof}
    A length-2 walk in $\mathcal{B}_r$ from a cell node $u$ to a cell node $v$ exists iff $u$ and $v$ share a row hyperedge. The matrix $H_r H_r^\top$ has $(u, v)$ entry equal to the number of shared incident row hyperedges, which is $1$ if $u$ and $v$ are in the same row and $0$ otherwise. Multiplication by $H_r D_e^{-1}H_r^\top$ therefore computes a degree normalized sum or mean over cells in the same row.
\end{proof}

Thus, a dual row--column layer provides full receptive fields for both row and column contexts in one encoder layer.

\subsection{Index Relabeling Invariance}\label{subsec:theo-consistency}

\uline{Key Result}. The stochastic bucket-index permutation and the consistency regularizer make the predictor insensitive to the arbitrary identities of bucket indices. The consistency loss equals twice the prediction variance over index relabelings, and minimizing it drives invariance to bucket-index permutations.

\smallskip\noindent\emph{Stochastic relabeling.} Each value is hashed and bucketized to an index in $\{0, \dots, B-1\}$. During training, we sample a random permutation $\pi \in S_B$ and perform embedding lookup using $\pi(\mathrm{idx})$. For a fixed table pair $(T_1, T_2)$, we define the prediction under $\pi$ as $\hat{y}_{\pi} = f_{\Theta}(\pi \circ \mathrm{idx}(T_1),\, \pi \circ \mathrm{idx}(T_2))$. The Consistency Regularizer samples $\pi, \pi'$ independently and penalizes $(\hat{y}_{\pi} - \hat{y}_{\pi'})^2$.

\begin{theorem}[Consistency loss equals prediction variance over index relabelings]\label{theorem:conloss_eq_predvar}
    For any fixed input pair $(T_1, T_2)$, with $\pi, \pi'$ sampled independently and uniformly from $S_B$, $\mathbb{E}_{\pi, \pi'} [(\hat{y}_{\pi} - \hat{y}_{\pi'})^2] = 2 \mathrm{Var}_{\pi} [\hat{y}_{\pi}]$.
\end{theorem}

\begin{proof}
    Let $X=\hat{y}_\pi$ and $Y=\hat{y}_{\pi'}$. Since $X$ and $Y$ are i.i.d.,
    $\mathbb{E}[(X-Y)^2]
    = \mathbb{E}[X^2] + \mathbb{E}[Y^2] - 2\mathbb{E}[X]\mathbb{E}[Y]
    = 2(\mathbb{E}[X^2] - (\mathbb{E}[X])^2)
    = 2\,\mathrm{Var}(X)$.
\end{proof}

Therefore, minimizing the consistency loss directly suppresses $\mathrm{Var}_\pi[\hat{y}_\pi]$, making predictions insensitive to the particular bucket indices assigned to values. Intuitively, because bucket identities are randomized across training steps, the model cannot attach stable semantics to any specific index and is pushed to rely on index-invariant signals.

\subsection{End-to-End Complexity}\label{subsec:theo-complexity}
We analyze the end-to-end cost of a deployed estimation from raw tables to predicted overlap ratio $\hat{\theta}$. Encoding a table $T$ with $m$ rows, $n$ columns, and $|T| = mn$ cells costs $\mathcal{O}(L |T| d)$ time and $\mathcal{O}(|T| d)$ space, dominated by the Row--Column Hypergraph Encoder. For a table pair $(T_1,\,T_2)$, interaction feature computation is linear in $|T_1| + |T_2|$, and the overlap fusion head with overlap ratio regressor are $\mathcal{O}(d)$ in the embedding dimension and independent of table~size.

Thus, for fixed $L$ and $d$, \Our{} achieves linear encoding and pairwise inference, matching the asymptotic behavior of \Arma{}, while avoiding exponential search over row and column permutations in \Sloth{}. The alignment and consistency regularizers are used only during training. Hence, inference retains the same linear time as a standard bi-encoder, making \Our{} suitable for large-scale discovery and matching workloads over millions of tables.

\section{Experiments}\label{sec:exp}

This section presents a comprehensive evaluation of our proposed method. We begin by outlining the experimental setup (Section~\ref{subsec:exp-setup}). We then study effectiveness in three settings: accuracy of overlap ratio estimation (Section~\ref{subsec:exp-indomain}), domain-robust generalization (Section~\ref{subsec:exp-crossdomain}), and query-by-table retrieval using overlap ratio estimates (Section~\ref{subsec:exp-retrieval}). Finally, we report efficiency (Section~\ref{subsec:exp-efficiency}) and ablations of key components (Section~\ref{subsec:exp-ablation}).

\subsection{Experimental Setup}\label{subsec:exp-setup}

\noindent\textbf{Datasets.}
We evaluate on three datasets. Two are benchmark corpora released by \arma{}~\cite{pugnaloni2025table}: \emph{\wiki{}} and \emph{\git{}}. They are derived from the Wikipedia table corpus~\cite{bleifuss2021structured, bleifuss2021secret} and GitHub repositories~\cite{hulsebos2023gittables}, containing~\numprint{128620} and~\numprint{256834} tables, respectively. We use the provided pairs and splits: \wiki{} has 500k/60k/60k train/validation/test pairs, and \git{} has 500k/100k/100k. In both datasets, pairs are balanced across ten overlap-ratio bins $\in [0,1]$. The third is \emph{\real{}}, derived from 2.1 million property records collected over time~\cite{li2018homeseeker}, \revtwo{containing \numprint{20000} tables and \numprint{10000} evaluation pairs.} 

\revtwo{These datasets differ in both table scale and value distributions. Figure~\ref{fig:table_dim_dist} shows the distributions of table dimensions: \wiki{} and \git{} consist of small-to-medium web tables, while \real{} lies at a much larger scale, with orders-of-magnitude larger areas. Table~\ref{tab:table_val_dist} profiles cell values at the table level for each dataset. The profiles show clear domain differences: \wiki{} has the highest value diversity and is mostly textual, \git{} contains substantially more floating-point values than the others, and \real{} has the lowest value diversity with lower variation across tables. We include \real{} as a zero-shot target to evaluate transfer beyond the two prior benchmark corpora, and describe its pairing procedure in Section~\ref{subsec:exp-crossdomain}. We exclude \real{} from training due to the substantial overhead of training on much larger tables.}

\begin{figure}[t]
    \centering
    \includegraphics[width=0.48\textwidth]{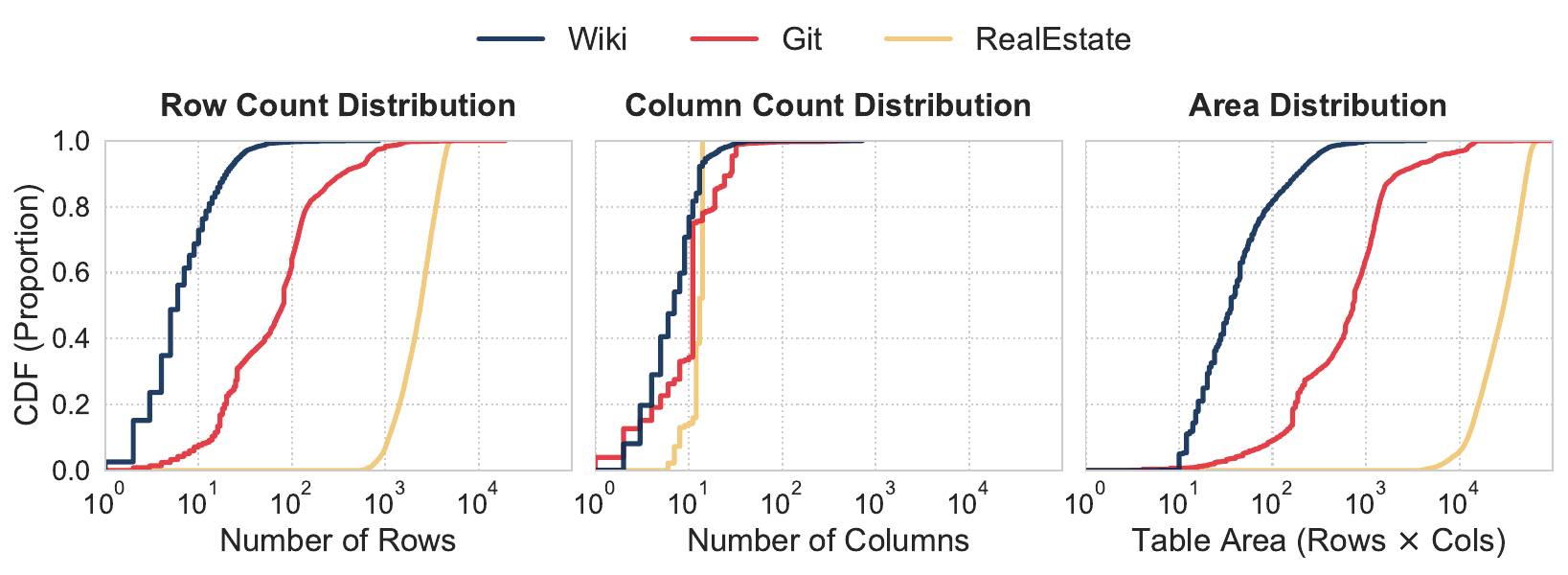}
    \caption[]{\revtwo{Cumulative distribution function of table dimensions across the three datasets.}}
    \label{fig:table_dim_dist}
\end{figure}
\begin{table}[t]
    \centering
    \footnotesize
    \setlength{\tabcolsep}{3pt}
    \caption{\revtwo{Profiles of cell-value distributions across the three datasets. ``Distinct'' is the fraction of distinct values among non-missing cells. ``Missing'', ``Text'', ``Int'', and ``Float'' are fractions over all cells in the table. Each entry reports the mean and standard deviation of these per-table percentages within each dataset.}}
    \label{tab:table_val_dist}
    \begin{tabular}{lccccc}
        \toprule
        \multirow{2}{*}{\textbf{Dataset}}
        & \multicolumn{2}{c}{\textbf{Value Profile}}
        & \multicolumn{3}{c}{\textbf{Value Type}} \\
        \cmidrule(lr){2-3}\cmidrule(lr){4-6}
        & \textbf{Distinct}
        & \textbf{Missing}
        & \textbf{Text}
        & \textbf{Int}
        & \textbf{Float} \\
        \midrule
        \wiki{} & \sdvalpct{65.5}{24.7} & \sdvalpct{7.3}{14.8}  & \sdvalpct{56.4}{29.6} & \sdvalpct{35.2}{30.6} & \sdvalpct{1.0}{4.3} \\
        \git{}  & \sdvalpct{51.0}{24.4} & \sdvalpct{11.4}{14.8} & \sdvalpct{41.0}{26.3} & \sdvalpct{32.3}{24.1} & \sdvalpct{15.2}{18.4} \\
        \real{} & \sdvalpct{37.1}{12.1} & \sdvalpct{12.0}{6.4}  & \sdvalpct{50.7}{11.3} & \sdvalpct{26.6}{8.8}  & \sdvalpct{10.8}{5.1} \\
        \bottomrule
    \end{tabular}
\end{table}


\smallskip\noindent\textbf{Methods for Comparison.}
We compare \our{} with the state of the art, \arma{}~\cite{pugnaloni2025table}, and its baselines, including exact and heuristic methods (\sloth{}, \js{}) and learned table embedding models (\bert{}, \rob{}, \turl{}, \embdi{}, \arma{}). We exclude LLM baselines because table serialization makes inference prohibitively expensive at our scale, and they do not naturally capture the injective row--column alignments that define overlap.

\begin{itemize}[leftmargin=*]
    \item \textbf{\sloth{}}~\cite{zecchini2024determining} computes ground-truth overlap ratios. We run its exact solver with a 60\,s timeout and fall back to the greedy approximation on timeout.
    
    \item \textbf{\js{}} variants measure value-level overlap between tables $T_1$ and $T_2$. Let $\mathcal{S}(T)$ be the set of distinct cell values in $T$, and let $\mathcal{B}(T)$ be the multiset of values including duplicates. We report:
    (i)~\emph{set-based} \textbf{\js{}}: $\frac{|\mathcal{S}(T_1)\cap\mathcal{S}(T_2)|}{|\mathcal{S}(T_1)\cup\mathcal{S}(T_2)|}$,
    (ii)~\emph{bag-union} variant (\textbf{\texttt{-BU}}): $\frac{2|\mathcal{B}(T_1)\cap\mathcal{B}(T_2)|}{|\mathcal{B}(T_1)|+|\mathcal{B}(T_2)|}$, and
    (iii)~\emph{bag-normalized} variant (\textbf{\texttt{-BN}}): $\frac{|\mathcal{B}(T_1)\cap\mathcal{B}(T_2)|}{\min(|T_1|,|T_2|)}$.
    
    \item \textbf{\bert{}/\rob{}}~\cite{devlin2019bert, liu2019roberta} are adapted via table-to-text serialization, each with three variants:
    (i)~\emph{Row} (\textbf{\texttt{-R}}) encodes each row as a sentence by concatenating cell values, then averages row embeddings to form a table vector;
    (ii)~\emph{Table} (\textbf{\texttt{-T}}) serializes the entire table into a single text sequence by joining rows and averages token embeddings; and
    (iii)~\emph{Hashed-Table} (\textbf{\texttt{-HT}}) applies the Table variant after replacing each cell value with its SHA-256 hash.
    
    \item \textbf{\turl{}}~\cite{deng2022turl} learns contextualized cell representations via a cell-filling objective and averages them into a table embedding.
    
    \item \textbf{\embdi{}}~\cite{cappuzzo2020embdi} adapts node2vec~\cite{grover2016node2vec} to a tripartite graph. Its table representations are pair-dependent and cannot be precomputed.
    
    \item \textbf{\arma{}}~\cite{pugnaloni2025table} embeds each table as a tripartite graph using GraphSAGE and estimates overlap via cosine similarity.
    
    \item \textbf{\our{}} is our alignment-guided overlap ratio estimator.
\end{itemize}

\smallskip\noindent\textbf{Evaluation Metrics.}
For pairwise prediction, we report mean absolute error (MAE) between predicted and ground-truth overlap ratio. For retrieval, we rank candidate tables by predicted overlap ratio and report nDCG@$k$~\cite{jarvelin2002ndcg}. For threshold-based filtering, we vary the overlap ratio threshold $\tau$ and treat a candidate as relevant if its ground-truth overlap ratio satisfies $\theta \ge \tau$. We report Selection Rate (proportion of tables returned), Recall, Precision, and F1. These metrics match those in \arma{}~\cite{pugnaloni2025table}. Specifically, Selection Rate is the complement of the Reduction Ratio, and Recall is equivalent to Pair Completeness~\cite{michelson2006rrpc, elfeky2002tailor}. Throughout the paper, $\downarrow$ indicates that lower values are better and $\uparrow$ indicates that higher values are better. We repeat each learning-based experiment five times with different random seeds and report mean $\pm$ standard deviation.

\smallskip\noindent\textbf{Training Details.}
\our{} is optimized using Adam with a learning rate of~$10^{-3}$. We train for up to 100~epochs with early stopping, using batch sizes of~512 on \wiki{} and~64 on \git{}. On a single GPU, training takes 30~hours on \wiki{} and 92~hours on \git{}. We set hash bucket size $B = 2 \times 10^6$ with embedding dimension of $d=256$. The model uses $L=2$ layers and a hidden dimension of $d_h=64$. We minimize MAE with alignment regularization ($\mathcal{\lambda}_{\text{RC}}=0.1$, $\mathcal{\lambda}_{\text{val}}=0.05$, $\mathcal{\lambda}_{\text{ctx}}=0.05$) and consistency regularization ($\lambda_{\text{cons}}=1.0$). For all learning-based baselines, we follow the training configuration described in~\cite{pugnaloni2025table}.

\smallskip\noindent\textbf{Implementation.}
All experiments were conducted on a Linux server with Intel Xeon~E5 CPUs (56 cores),~512\,GB of RAM and an NVIDIA Tesla~P100 GPU (16\,GB VRAM). The implementation, developed in Python and PyTorch, is available at~\cite{code2025}.

\subsection{Accuracy of Overlap Ratio Estimation}\label{subsec:exp-indomain}


\begin{table}[t]
    \centering
    \footnotesize
    
    \caption{MAE ($\downarrow$) for overlap ratio estimation. \hlBest{Dark blue} and \hlSec{light blue} highlight the best and second-best performers. ``\oot{}'' indicates out of time.}
    \label{tab:mae_in}
    
    \begin{threeparttable}
        \begin{tabular}{l c c}
            \toprule
            \textbf{Method} & \textbf{\wiki{}} & \textbf{\git{}} \\
            \midrule
            
            \js{}     & 0.223 & 0.151 \\ 
            \jbu{}    & 0.231 & 0.157 \\
            \jbn{}    & 0.278 & 0.173 \\
            \bertR{}  & \sdval{0.115}{0.0027} & \sdval{0.269}{0.0115} \\
            \bertT{}  & \sdval{0.097}{0.0049} & \sdval{0.127}{0.0036} \\
            \bertHT{} & \sdval{0.207}{0.0009} & \sdval{0.163}{0.0035} \\
            \robR{}   & \sdval{0.135}{0.0045} & \sdval{0.225}{0.0185} \\
            \robT{}   & \sdval{0.101}{0.0042} & \sdval{0.128}{0.0058} \\
            \robHT{}  & \sdval{0.229}{0.0637} & \sdval{0.137}{0.0007} \\
            \turl{}   & \sdval{0.199}{0.0035} & \sdval{0.275}{0.0117} \\
            \embdi{}  & \sdval{0.372}{0.0002} & \oot{}                   \\
            \arma{}   & \hlSec{\sdval{0.069}{0.0006}} & \hlSec{\sdval{0.066}{0.0011}} \\
            \textbf{\our{}} & \hlBest{\sdval{0.040}{0.0018}} & \hlBest{\sdval{0.030}{0.0004}} \\
            
            \bottomrule
        \end{tabular}
    \end{threeparttable}
\end{table}

We first evaluate how accurately each method estimates overlap ratios between individual pairs of tables using the MAE. All methods are tested on held-out pairs from the same dataset. For learning-based approaches, models are trained and tested on disjoint splits, while non-learning baselines are directly applied to the same test pairs for comparability.

Table~\ref{tab:mae_in} shows that \our{} achieves the lowest MAE on both datasets. \embdi{} exceeds the runtime of the exact solver on \git{} and is therefore marked as ``\oot{}''. \our{} improves over \arma{} from $0.069$ to $0.040$ on \wiki{} and from $0.066$ to $0.030$ on \git{}, indicating that the gain is not limited to a single corpus. Compared with set-based baselines, \js{} remains competitive on \git{} but degrades substantially on \wiki{}, where ignoring row--column structure leads to large errors. Transformer-based table embeddings (e.g., \bertT{}/\robT{}/\turl{}) provide mixed benefits and are less reliable across datasets, suggesting that semantic priors are not consistently aligned with the strictly equality-based overlap definition.

Figure~\ref{fig:mae_dist} further breaks down MAE by overlap ratio interval. The largest gaps appear in the low-to-mid overlap range, where accurate estimation requires recovering partial row--column correspondences rather than detecting near-duplicates. In these intervals, \our{} consistently achieves the lowest error on both \wiki{} and \git{}. As $\theta$ approaches $1$, all methods improve because highly overlapping pairs are closer to duplication and are easier to distinguish.

\begin{figure}[t]
    \centering
    \includegraphics[width=0.48\textwidth]{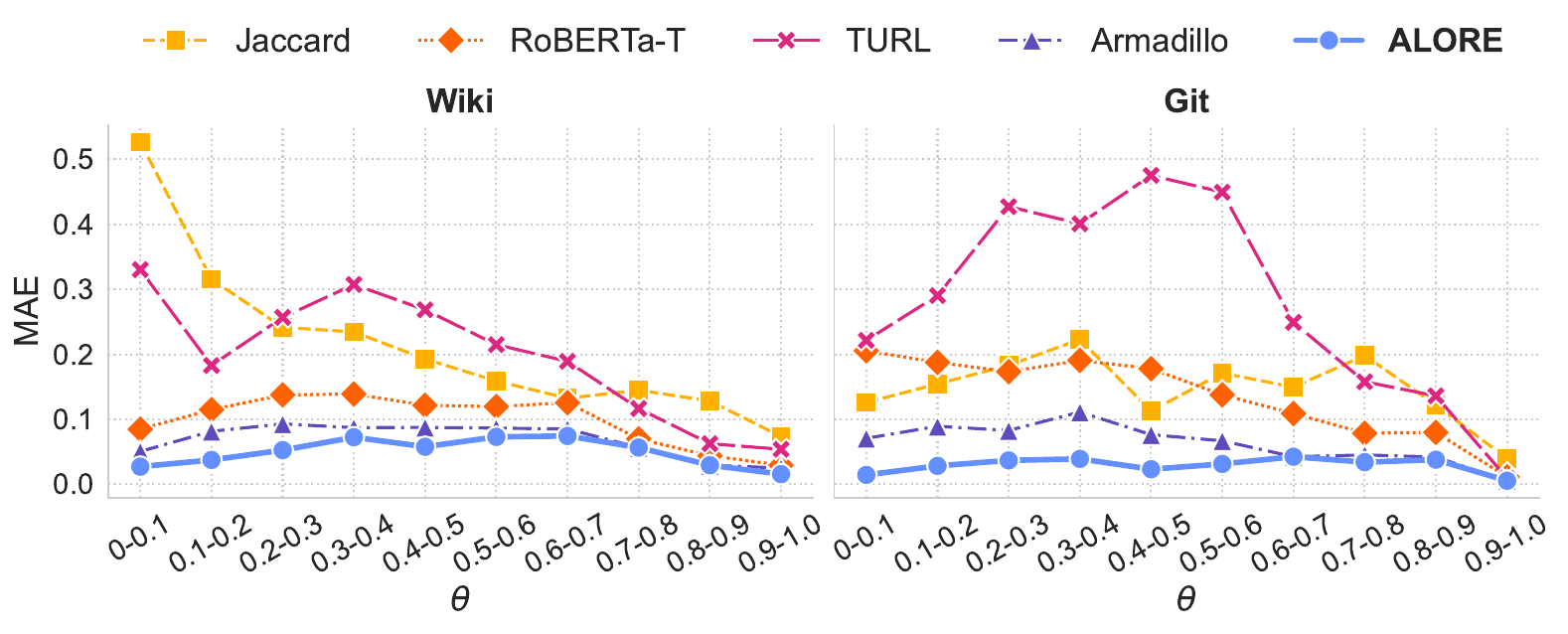}
    \caption[]{MAE ($\downarrow$) of overlap ratio estimation across different ranges of overlap ratio $\theta$.}
    \label{fig:mae_dist}
\end{figure}

\subsection{Domain-Robust Generalization}\label{subsec:exp-crossdomain}

To assess generalization, learning-based methods are trained on one dataset and evaluated on another. This is a strict \emph{zero-shot} setting where no target-domain fine-tuning is performed. Beyond \wiki{} and \git{}, we include \real{} as a target domain with different value distributions and much larger tables (Table~\ref{tab:table_val_dist} and Figure~\ref{fig:table_dim_dist}). \real{} is derived from 2.1 million property sale records: each table is a time-bounded transaction snapshot, where rows are properties and columns include address, sale date, and price. We form pairs by sampling two snapshots that partially share properties, simulating integration or deduplication across extracts (e.g., different months). For shared properties, attributes like sale date and price may differ due to repeated sales and evolving listings, yielding partial overlaps between snapshots. We construct \numprint{10000} such pairs from \numprint{20000} tables. The \real{} dataset is available~at~\cite{code2025}.

\revtwo{Table~\ref{tab:mae_cross} shows that \our{} transfers more reliably than other methods across the evaluated source--target domain pairs.} Relative to \arma{}, \our{} yields consistent improvements across all source--target pairs, with the largest gains on \real{} as the target domain (up to $69\%$) and a substantial improvement for \wiki{}$\rightarrow$\git{} ($38\%$). Notably, \our{} remains strong when transferring from a smaller source domain to a larger or more heterogeneous target, including \wiki{}$\rightarrow$\git{}/\real{} and \git{}$\rightarrow$\real{}. In contrast, semantic table encoders are inconsistent across targets. \turl{} is competitive on \real{} but still trails \our{}. Overall, these results match our design goal of domain robustness. \revtwo{By avoiding reliance on corpus-specific value identities and emphasizing permutation-invariant structural signals, \our{} better preserves accuracy under the evaluated zero-shot domain shifts.}


\begin{table}[t]
    \centering
    \footnotesize
    
    \caption{MAE ($\downarrow$) for cross-domain evaluation. Results are grouped by target domain (test set) for models trained on different source domains (training set).}
    \label{tab:mae_cross}
    
    \begin{threeparttable}
        \begin{tabular}{l c c c c}
            \toprule
            
            \multirow{2}{*}{\textbf{Method}} & \textbf{Target: \wiki{}} & \textbf{Target: \git{}} & \multicolumn{2}{c}{\textbf{Target: \real{}}} \\
            
            \cmidrule(lr){2-2} \cmidrule(lr){3-3} \cmidrule(lr){4-5}
             
             & \textit{Source: \git{}} & \textit{Source: \wiki{}} & \textit{Source: \wiki{}} & \textit{Source: \git{}} \\
            
            \midrule
            
            \bertR{}  & \sdval{0.362}{0.0013} & \sdval{0.395}{0.0007} & \revtwo{\sdval{0.470}{0.0001}} & \revtwo{\sdval{0.471}{0.0001}} \\
            \bertT{}  & \sdval{0.306}{0.0018} & \sdval{0.296}{0.0018} & \revtwo{\sdval{0.297}{0.0037}} & \revtwo{\sdval{0.331}{0.0035}} \\
            \bertHT{} & \sdval{0.294}{0.0019} & \sdval{0.304}{0.0022} & \revtwo{\sdval{0.244}{0.0114}} & \revtwo{\sdval{0.266}{0.0106}} \\
            \robR{}   & \sdval{0.354}{0.0011} & \sdval{0.392}{0.0022} & \revtwo{\sdval{0.467}{0.0002}} & \revtwo{\sdval{0.467}{0.0004}} \\
            \robT{}   & \sdval{0.319}{0.0023} & \sdval{0.300}{0.0027} & \revtwo{\sdval{0.290}{0.0020}} & \revtwo{\sdval{0.271}{0.0029}} \\
            \robHT{}  & \sdval{0.305}{0.0009} & \sdval{0.328}{0.0576} & \revtwo{\sdval{0.251}{0.0194}} & \revtwo{\sdval{0.293}{0.0144}} \\
            \turl{}   & \sdval{0.354}{0.0065} & \sdval{0.320}{0.0019} & \hlSec{\revtwo{\sdval{0.241}{0.0009}}} & \hlSec{\revtwo{\sdval{0.265}{0.0025}}} \\
            \arma{}   & \hlSec{\sdval{0.211}{0.0116}} & \hlSec{\sdval{0.247}{0.0092}} & \revtwo{\sdval{0.465}{0.0022}} & \revtwo{\sdval{0.462}{0.0120}} \\
            \textbf{\our{}} & \hlBest{\sdval{0.208}{0.0101}} & \hlBest{\sdval{0.154}{0.0148}} & \hlBest{\revtwo{\sdval{0.183}{0.0345}}} & \hlBest{\revtwo{\sdval{0.144}{0.0315}}} \\
            
            \bottomrule
        \end{tabular}
    \end{threeparttable}
\end{table}

\begin{table}[t]
    \centering
    \footnotesize
    
    \caption{Ranking performance on \gitq{} measured by nDCG@k ($\uparrow$).}
    \label{tab:rank_ndcg}
    
    \begin{threeparttable}
        \begin{tabular}{l c c c c}
            \toprule
            
            \textbf{Method} & \textbf{nDCG@1} & \textbf{nDCG@10} & \textbf{nDCG@50} & \textbf{nDCG@100} \\
            
            \midrule
            
            \js{}       & \hlSec{0.858} & 0.811 & 0.807 & 0.790 \\
            \jbu{}      & 0.839 & 0.813 & 0.808 & 0.794 \\
            \jbn{}      & 0.782 & \hlSec{0.821} & \hlSec{0.853} & \hlSec{0.861} \\
            
            \bertR{}    & \sdval{0.809}{0.003} & \sdval{0.761}{0.004} & \sdval{0.719}{0.008} & \sdval{0.680}{0.012} \\
            \bertT{}    & \sdval{0.796}{0.004} & \sdval{0.760}{0.002} & \sdval{0.749}{0.001} & \sdval{0.716}{0.002} \\
            \bertHT{}   & \sdval{0.698}{0.009} & \sdval{0.680}{0.003} & \sdval{0.694}{0.003} & \sdval{0.672}{0.003} \\
            \robR{}     & \sdval{0.803}{0.007} & \sdval{0.762}{0.004} & \sdval{0.729}{0.004} & \sdval{0.695}{0.004} \\
            \robT{}     & \sdval{0.817}{0.003} & \sdval{0.763}{0.002} & \sdval{0.735}{0.004} & \sdval{0.702}{0.005} \\
            \robHT{}    & \sdval{0.701}{0.005} & \sdval{0.681}{0.003} & \sdval{0.695}{0.002} & \sdval{0.680}{0.002} \\
            \turl{}     & \sdval{0.719}{0.007} & \sdval{0.709}{0.004} & \sdval{0.693}{0.003} & \sdval{0.672}{0.004} \\
            \arma{}     & \sdval{0.765}{0.014} & \sdval{0.732}{0.003} & \sdval{0.708}{0.005} & \sdval{0.675}{0.008} \\
            \textbf{\our{}} & \hlBest{\sdval{0.881}{0.006}} & \hlBest{\sdval{0.870}{0.004}} & \hlBest{\sdval{0.878}{0.004}} & \hlBest{\sdval{0.862}{0.003}} \\
            
            \bottomrule
        \end{tabular}
    \end{threeparttable}
\end{table}

\subsection{Query-by-Table Retrieval}\label{subsec:exp-retrieval}

We evaluate how overlap ratio estimation supports retrieval over a large table repository. This setting corresponds to the \emph{table querying} scenario in~\cite{pugnaloni2025table}, where given a query table, the system retrieves overlapping tables either by ranking candidates (Section~\ref{subsubsec:exp-rank}) or by filtering with an overlap threshold (Section~\ref{subsubsec:exp-threshold}).

\smallskip\noindent\textbf{Setup.}
The setting consists of a repository of \numprint{10000} tables and a query set of 100 tables, sampled from the test split of \git{}~\cite{pugnaloni2025table} (we denote this benchmark as \gitq{}). For each query, each method predicts overlap ratios for all \numprint{10000} candidates and ranks them by the predicted scores.
The data are highly skewed, where most query--candidate pairs have near-zero overlap, with only a small fraction exhibiting moderate-to-high overlap (distributions are plotted in Appendix~\ref{app:exp-retrieval}). As a result, retrieval quality depends on separating a few truly overlapping tables from many near-zero candidates.
We also evaluate on the TP-TR (small) benchmark~\cite{fan2024gent}: the query set contains 26 TP-TR tables, and the candidate set consists of \wiki{} plus another 32 TP-TR tables. The results are deferred to Appendix~\ref{app:tptr}. 

\subsubsection{Ranking evaluation}\label{subsubsec:exp-rank}

We measure ranking quality using nDCG at cutoff $k$~\cite{jarvelin2002ndcg}, which evaluates how well the predicted ranking matches the ideal ranking at the top of the list. Table~\ref{tab:rank_ndcg} reports nDCG at multiple cutoffs $k$. Small $k$ is included because the query--candidate overlap distribution is highly skewed, as shown in Appendix Figure~\ref{fig:query_olap_dist}, and top-ranked results are the main focus in practice. Overall, \our{} achieves the best ranking quality across all $k$, indicating that it more effectively promotes truly high-overlap tables to the top. Simple Jaccard variants are strong baselines, consistent with prior observations that many high-overlap cases correspond to duplicate or containment relationships where structural modeling is less critical than capturing shared values~\cite{pugnaloni2025table}. Despite this, \our{} still improves the top-$k$ ordering, suggesting that its structure-aware encoding and alignment-guided training better resolve the fine-grained score differences that matter for retrieval.

\subsubsection{Threshold-based retrieval}\label{subsubsec:exp-threshold}

\begin{figure}[t]
    \centering
    \includegraphics[width=0.48\textwidth]{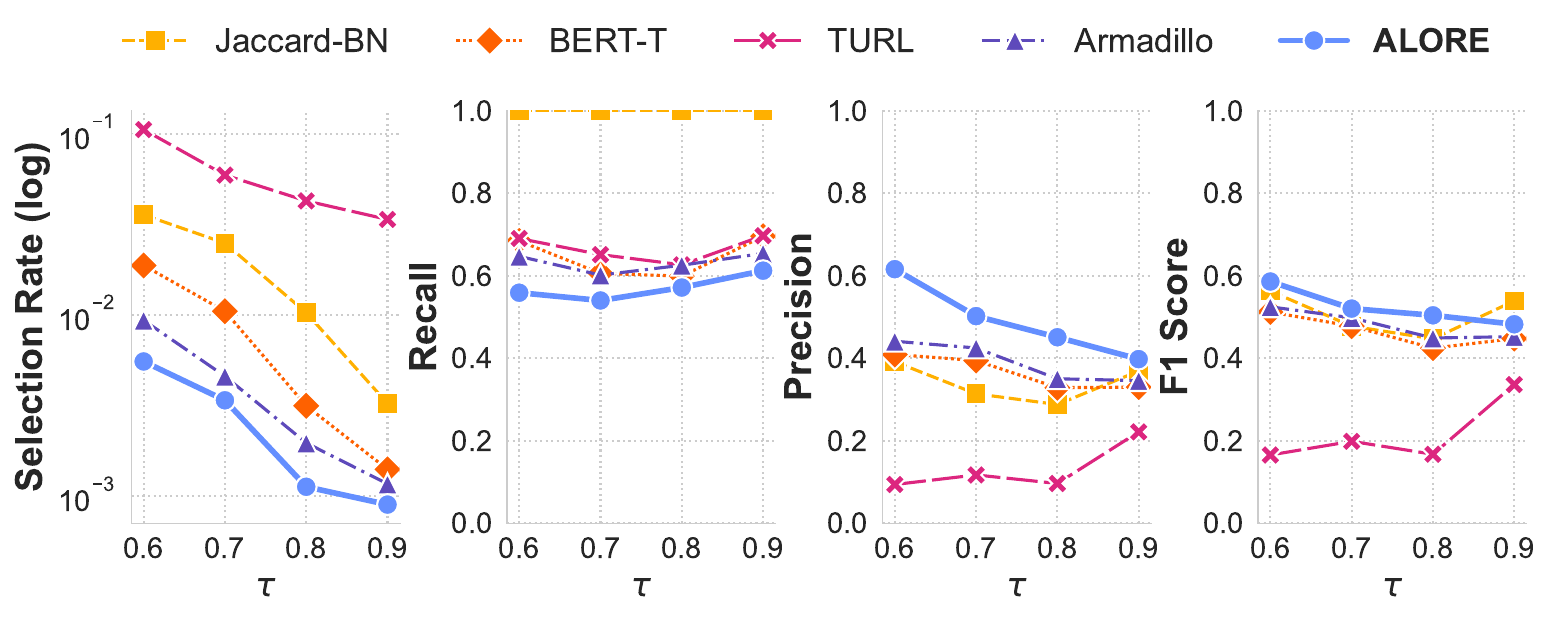}
    \caption[]{\revthree{Threshold-based retrieval performance on \gitq{} across overlap ratio thresholds $\tau$. We report Selection Rate ($\downarrow$), Recall ($\uparrow$), Precision ($\uparrow$), and F1 Score ($\uparrow$).}}
    \label{fig:query_olap_sel}
\end{figure}

Beyond ranking, many pipelines require filtering candidates by a minimum overlap to obtain a small, high-quality candidate subset. This candidate subset can then be passed to expensive downstream processing such as exact overlap computation~\cite{zecchini2024determining}, table reclamation~\cite{fan2024gent}, related table discovery~\cite{das2012finding}, or deduplication~\cite{koch2023duplicate}. We therefore also report threshold-based retrieval curves (Figure~\ref{fig:query_olap_sel}) by varying the overlap threshold $\tau$ and measuring how effectively each method prunes the candidate set while retaining tables with $\theta \ge \tau$. We focus on $\tau \in \{0.6,\,0.7,\,0.8,\,0.9\}$ because (i)~most query-candidate pairs have small overlaps (Figure~\ref{fig:query_olap_dist}), where low thresholds make filtering less discriminative; and (ii)~practical use cases emphasize highly overlapping tables, where compact high-overlap candidate subsets reduce verification cost and user effort.

Across all thresholds, \our{} achieves the lowest selection rate, returning the smallest candidate subset per query. It also consistently attains the highest precision for every $\tau$ and the best F1 for $\tau \in \{0.6,\,0.7,\,0.8\}$, demonstrating a strong effectiveness--cost trade-off. \revthree{The baselines achieve higher recall by selecting larger candidate subsets, but this comes with correspondingly lower precision, which increases downstream cost. Thus, Figure~\ref{fig:query_olap_sel} should be interpreted as a trade-off among selection rate, recall, and precision in threshold-based retrieval. \our{} is most suitable when downstream cost favors compact and precise candidate sets, whereas recall-first blocking pipelines may prefer a less selective method such as \arma{}.}

\revthree{This trade-off can occur despite \our{}'s strong overall estimation accuracy because hard-threshold filtering is sensitive to score calibration near the cutoff $\tau$. Small underestimation around the threshold can turn relevant tables into false negatives. Thus, low global regression error does not necessarily imply the highest recall at every threshold. Adapting \our{} for recall-first filtering would require less selective filtering near the cutoff, for example by favoring recall in threshold tuning or by adding a stronger penalty for false negatives near the target threshold. This limitation is specific to hard-threshold filtering. For ranking-based retrieval, \our{} remains strong and achieves the best ranking quality across all $k$.}

At $\tau=0.9$, \jbn{} achieves slightly higher F1, consistent with prior observations that these high-overlap pairs are near duplicates where shared values alone provide a strong signal~\cite{pugnaloni2025table}. However, it does so by returning a much larger subset (an order of magnitude larger than \our{}), whereas \our{} remains more cost-effective by keeping the subset compact.

\subsection{Efficiency Analysis}\label{subsec:exp-efficiency}

\begin{figure}[t]
    \centering
    \includegraphics[width=0.48\textwidth]{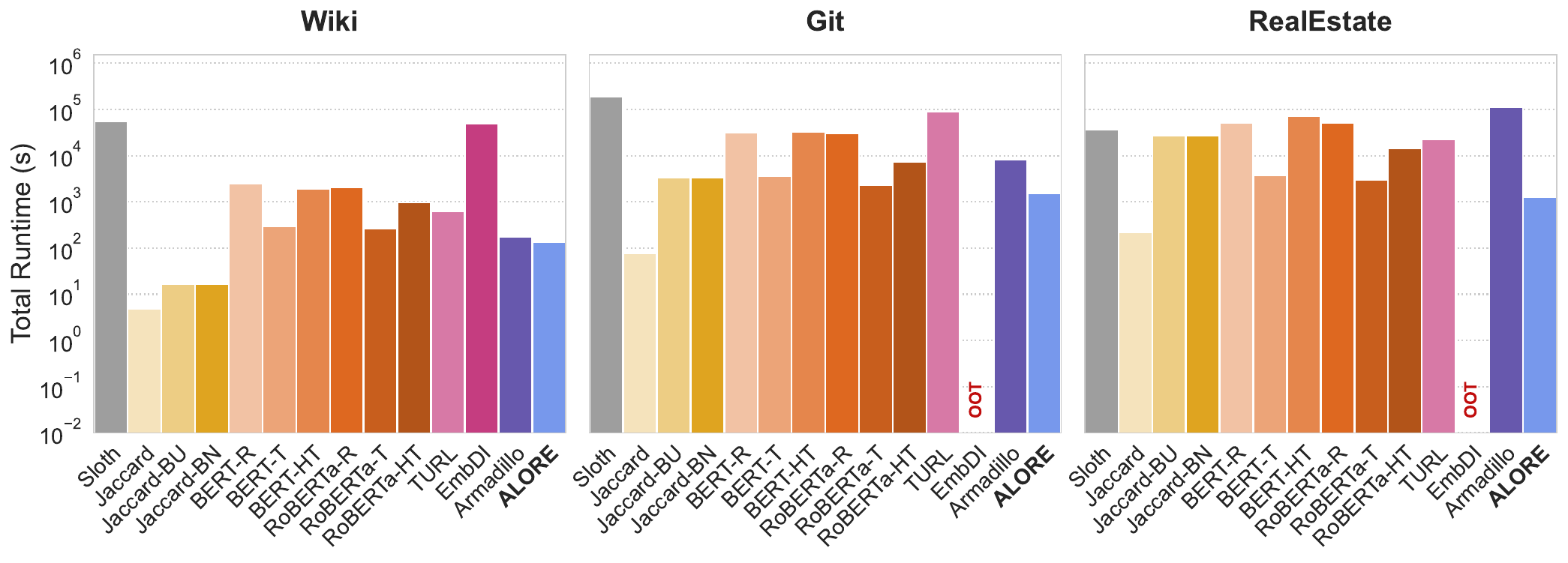}
    \caption[]{\revtwo{End-to-end runtime summed over all pairs in seconds ($\downarrow$). \sloth{} exact times out on 15\% (\wiki{}), 28\% (\git{}), and 1.3\% (\real{}) of pairs and falls back to greedy.}}
    \label{fig:efficiency}
\end{figure}

\begin{table}[t]
    \centering
    \footnotesize 
    \setlength{\tabcolsep}{1pt}
    \sisetup{
        detect-all, 
        group-minimum-digits=4,
        group-separator={\,}
    }
    
    \begin{threeparttable}
        \caption{\revtwo{Mean runtime breakdown in milliseconds ($\downarrow$). Graph construction (\emph{Graph}) and embedding generation (\emph{Embed}) are averaged per table. Inference (\emph{Infer}) is averaged per pair.}}
        \label{tab:efficiency}

        \newcommand{\mc}[1]{\multicolumn{1}{c}{#1}}

        \begin{tabular}{
            l 
            S[table-format=3.1] 
            S[table-format=4.1] 
            S[table-format=3.2] 
            c 
            S[table-format=3.1] 
            S[table-format=4.1] 
            S[table-format=4.2] 
            c 
            S[table-format=4.0] 
            S[table-format=4.1] 
            S[table-format=3.2]
        }
            \toprule
            \multirow{2}{*}{\textbf{Method}} &
            \multicolumn{3}{c}{\textbf{\wiki{}}} & &
            \multicolumn{3}{c}{\textbf{\git{}}} & &
            \multicolumn{3}{c}{\textbf{\real{}}\textsuperscript{\textdagger}} \\
            
            \cmidrule(lr){2-4} \cmidrule(lr){6-8} \cmidrule(lr){10-12}
            
            & \mc{\textit{Graph}} & \mc{\textit{Embed}} & \mc{\textit{Infer}} &
            & \mc{\textit{Graph}} & \mc{\textit{Embed}} & \mc{\textit{Infer}} &
            & \mc{\textit{Graph}} & \mc{\textit{Embed}} & \mc{\textit{Infer}} \\
            
            \midrule

            \sloth{}    & {--} & {--} & 919.76 & & {--} & {--} & 1905.39 & & {--} & {--} & 3646.46 \\
            \js{}       & {--} & {--} & 0.08   & & {--} & {--} & 0.77    & & {--} & {--} & 21.92 \\
            \jbu{}      & {--} & {--} & 0.27   & & {--} & {--} & 33.35   & & {--} & {--} & 2774.64 \\
            \jbn{}      & {--} & {--} & 0.27   & & {--} & {--} & 33.35   & & {--} & {--} & 2774.64 \\
            \bertR{}    & {--} & 127.7  & 0.13 & & {--} & 668.6  & 0.14 & & {--} & 2541.9 & 0.19 \\
            \bertT{}    & {--} & 14.5   & 0.13 & & {--} & 76.6   & 0.13 & & {--} & 189.3  & 0.19 \\
            \bertHT{}   & {--} & 97.2   & 0.13 & & {--} & 701.6  & 0.14 & & {--} & 3566.0 & 0.20 \\
            \robR{}     & {--} & 146.9  & 0.13 & & {--} & 643.5  & 0.14 & & {--} & 2566.6 & 0.21 \\
            \robT{}     & {--} & 13.2   & 0.13 & & {--} & 48.9   & 0.14 & & {--} & 150.6  & 0.20 \\
            \robHT{}    & {--} & 49.0   & 0.13 & & {--} & 154.8  & 0.14 & & {--} & 711.0  & 0.19 \\
            \turl{}     & {--} & 30.9   & 0.13 & & {--} & 1876.9 & 0.13 & & {--} & 11.9   & 0.34 \\
            \embdi{}    & 388.6 & 2149.5 & 0.17 & & {\oot{}} & {\oot{}} & {\oot{}} & & {\oot{}} & {\oot{}} & {\oot{}} \\
            \arma{}     & 7.3   & 1.3    & 0.10 & & 173.2    & 1.1      & 0.11     & & 5666 & 12.0 & 0.13 \\
            \textbf{\our{}} & \bfseries 1.7 & \bfseries 5.0 & \bfseries 0.07 & & \bfseries 28.4 & \bfseries 3.5 & \bfseries 0.06 & & \bfseries 53 & \bfseries 11.0 & \bfseries 0.11 \\
            
            
            \bottomrule
        \end{tabular}
        \begin{tablenotes}[flushleft]
            \scriptsize
            \item[\textdagger]All learning models evaluated on \real{} are trained on \git{}.
        \end{tablenotes}
    \end{threeparttable}
\end{table}

Figure~\ref{fig:efficiency} shows the total end-to-end runtime (log scale) summed over all test pairs. \our{} is up to $89\times$ faster than the state-of-the-art learning-based estimator \arma{}. Across all datasets, \our{} is consistently faster than learning-based methods, with orders of magnitude gains over several baselines. For \sloth{}, we set a longer timeout for the exact solver than the default (60\,s vs.\ 3\,s). Even with this larger budget, the exact phase times out on 15\% of \wiki{} pairs, 28\% of \git{} pairs and 1.3\% of \real{} pairs, triggering the greedy approximation. Table~\ref{tab:efficiency} decomposes the end-to-end runtime into three components: (i)~\emph{graph/hypergraph construction}, (ii)~\emph{embedding generation}, and (iii)~\emph{inference}. They are reported as averages per pair and per table. For embedding-based methods, graph/hypergraph construction and embedding generation can be precomputed once per table and reused across many pair queries. The breakdown shows that embedding generation dominates the cost of learning-based baselines, while graph construction becomes the primary bottleneck for \arma{} due to its tripartite graph with explicit row, column, and cell nodes. By contrast, \our{} maintains low cost in both graph construction and embedding generation, reflecting the linear-time Two-View Row–Column Hypergraph construction and the scalable Row--Column Hypergraph Encoder. Inference cost is negligible for all learning-based methods, and \our{} incurs the smallest overhead since prediction reduces to a compact fusion head over fixed-size representations.

\subsection{Ablation Study}\label{subsec:exp-ablation}

We ablate major components of \our{} while training on \wiki{} and evaluating on \wiki{} (in-domain) and on \git{}/\real{} (cross-domain). Table~\ref{tab:ablation} shows that most removals increase MAE, with several effects more pronounced on cross-domain targets.
\revtwo{Removing inter-table interaction feature $I=[S\parallel J]$ degrades performance most strongly on \real{}, from $0.183$ to $0.304$, showing that these coarse pairwise features are useful under stronger domain shift. However, the interaction-only variant, which uses only $I$ without the structural encoder, performs much worse than both the full model and the no-interaction variant, reaching $0.242$ on \git{} and $0.332$ on \real{}. Thus, inter-table interaction feature $I$ acts as a useful calibration signal rather than a shortcut. It complements the learned structural representation with coarse value-overlap and size--shape signals for the final predictor, but it is not sufficient without the structure-aware encoder, alignment-guided training, and domain-robust value mapping.}

\revtwo{We further decompose $I$ to examine which interaction signal contributes more.} Within $I$, removing the Jaccard anchor $J$ has a larger effect than removing size–shape features $S$, especially on \git{} and \real{}. For the encoder, both row and column hyperedges contribute, while removing the multi-scale readout is the most detrimental encoder ablation on \real{}. Regarding alignment, removing the regularizer $\mathcal{L}_{\text{align}}$ increases MAE on all targets, with Context-Align $\mathcal{L}_{\text{ctx}}$ showing the largest impact on \real{}. Similarly, removing the consistency regularizer $\mathcal{L}_{\text{cons}}$ worsens cross-domain performance, notably on \real{}. Finally, disabling stochastic permutations $\pi$ improves \wiki{} but worsens \git{} and \real{}, indicating a trade-off between in-domain fit and cross-domain robustness.


\begin{table}[t]
    \centering
    \footnotesize
    \setlength{\tabcolsep}{1.5pt}
    \caption{\revtwo{Ablation study of \our{} trained on \wiki{}, in MAE ($\downarrow$).}}
    \label{tab:ablation}
    
    \begin{threeparttable}
        \begin{tabular}{l c c c}
            \toprule
            
            \multirow{2}{*}{\textbf{Method}} & \textbf{Target: \wiki{}} & \textbf{Target: \git{}} & \textbf{Target: RealEst} \\
            
             
             & \textit{Source: \wiki{}} & \textit{Source: \wiki{}} & \textit{Source: \wiki{}} \\
            
            \midrule
            
            \textbf{\our{} (Full model)} & \hlSec{\sdval{0.040}{0.0018}} & \hlBest{\sdval{0.154}{0.0148}} & \hlBest{\revtwo{\sdval{0.183}{0.0345}}} \\
            
            \midrule
            \multicolumn{4}{l}{\textit{Inter-Table Interaction}} \\
            \quad No interaction $I$ & \sdval{0.064}{0.0026} & \sdval{0.190}{0.0040} & \revtwo{\sdval{0.304}{0.0266}} \\
            \quad \quad No size-shape $S$ & \sdval{0.057}{0.0018} & \sdval{0.174}{0.0010} & \revtwo{\sdval{0.204}{0.0167}} \\
            \quad \quad No Jaccard $J$ & \sdval{0.058}{0.0020} & \sdval{0.183}{0.0171} & \revtwo{\sdval{0.222}{0.0087}} \\
            \revtwo{\quad Only interaction $I$} & \revtwo{\sdval{0.074}{0.0029}} & \revtwo{\sdval{0.242}{0.0192}} & \revtwo{\sdval{0.332}{0.0954}} \\
            
            \midrule
            \multicolumn{4}{l}{\textit{Hypergraph \& Encoder}} \\
            \quad No column hyperedges $\mathcal{E}_{\mathrm{col}}$ & \sdval{0.056}{0.0035} & \sdval{0.206}{0.0084} & \revtwo{\sdval{0.220}{0.0724}} \\
            \quad No row hyperedges $\mathcal{E}_{\mathrm{row}}$ & \sdval{0.053}{0.0009} & \sdval{0.172}{0.0220} & \revtwo{\sdval{0.201}{0.0699}} \\
            \quad No multi-scale readout\textsuperscript{\textdagger} & \sdval{0.062}{0.0027} & \sdval{0.195}{0.0131} & \revtwo{\sdval{0.259}{0.0362}} \\
            
            \midrule
            \multicolumn{4}{l}{\textit{Inter-Table Alignment}} \\
            \quad No alignment regularizer $\mathcal{L}_{\text{align}}$ & \sdval{0.057}{0.0026} & \sdval{0.168}{0.0282} & \revtwo{\sdval{0.210}{0.0859}} \\
            \quad \quad No RowColumn-Align $\mathcal{L}_{\text{RC}}$ & \sdval{0.051}{0.0012} & \sdval{0.161}{0.0023} & \revtwo{\sdval{0.211}{0.1333}} \\
            \quad \quad No Value-Align $\mathcal{L}_{\text{val}}$ & \sdval{0.051}{0.0008} & \sdval{0.168}{0.0068} & \hlSec{\revtwo{\sdval{0.193}{0.0301}}} \\
            \quad \quad No Context-Align $\mathcal{L}_{\text{ctx}}$ & \sdval{0.047}{0.0008} & \sdval{0.163}{0.0055} & \revtwo{\sdval{0.239}{0.0560}} \\
            \quad No gating $W_{\text{gate}}$ & \sdval{0.053}{0.0019} & \hlSec{\sdval{0.158}{0.0082}} & \revtwo{\sdval{0.295}{0.1555}} \\
            
            \midrule
            \multicolumn{4}{l}{\textit{Domain Robustness}} \\
            \quad No consistency regularizer $\mathcal{L}_{\text{cons}}$ & \sdval{0.052}{0.0034} & \sdval{0.169}{0.0139} & \revtwo{\sdval{0.223}{0.0071}} \\
            \quad No permutation $\pi$ (and $\mathcal{L}_{\text{cons}}$) & \hlBest{\sdval{0.036}{0.0001}} & \sdval{0.179}{0.0262} & \revtwo{\sdval{0.256}{0.0281}} \\
            
            \bottomrule
        \end{tabular}
        \begin{tablenotes}[flushleft]
            \scriptsize
            \item[\textdagger]No multi-scale readout uses only the last layer pooled embedding $z^{(L)}$.
        \end{tablenotes}
    \end{threeparttable}
\end{table}

\section{Related Work}

\smallskip\noindent\textbf{Table Overlap.}
The closest work to ours is \Arma{}~\cite{pugnaloni2025table}, which embeds each table as a graph and estimates overlap ratio through embedding similarity. Our work follows the same overlap definition, but develops an estimator that more directly models inter-table row/column interactions that determine overlap ratio. The overlap definition is formalized by \Sloth{}~\cite{zecchini2024determining}, which defines the largest overlap through an attribute mapping between column subsets and measures overlap as the area of the induced common subtable. \Sloth{} provides exact and greedy algorithms, but its combinatorial search is expensive at scale, which motivates overlap ratio estimation. Other related work studies duplicate tables across Wikipedia snapshots to track table evolution~\cite{bleifuss2021structured} and lake deduplication under exact match and containment~\cite{koch2023duplicate}.

\smallskip\noindent\textbf{Table Representation Learning.}
Table encoders have been widely studied for table understanding, search, and similarity. Transformer-based models such as TaBERT~\cite{yin2020tabert} and TURL~\cite{deng2022turl} learn contextualized representations for tabular content. Structure-aware methods incorporate row/column structure, such as StruBERT~\cite{trabelsi2022strubert} and the hypergraph-enhanced tabular language model HyTrel~\cite{chen2023hytrel}. EmbDI~\cite{cappuzzo2020embdi} learns graph embeddings from relational data for integration tasks. These methods target general purpose semantic representations, whereas we tailor representations and objectives to overlap ratio estimation with inter-table alignment signals and row/column permutation invariances.

\smallskip\noindent\textbf{Schema Matching and Entity Resolution.}
Schema matching aligns semantically equivalent attributes across tables~\cite{hong2002coma, aumueller2005schema, liu2025magneto}, while entity resolution matches tuples referring to the same real-world entity~\cite{wu2020zeroer, wang2023sudowoodo}. Both typically rely on schema metadata, entity features, domain knowledge, or textual semantics, and their outputs are matchings over attributes or entities. Our task instead predicts the overlap ratio derived from the largest common subtable.


\smallskip\noindent\textbf{Related-Table Discovery.}
Data discovery in data lakes~\cite{miller2018open, das2012finding, zhang2020finding, bogatu2020dataset, wang2023solo} is framed as unionable or joinable table search. Unionable search focuses on semantic column compatibility~\cite{nargesian2018table,khatiwada2023santos,fan2023semantics}, while joinable search focuses on key and value overlap with scalable indexing or hashing~\cite{zhu2019josie,dong2021efficient,esmailoghli2022mate}. In contrast, overlap ratio estimation is defined by the largest common subtable and depends on row--column structure beyond key overlap or semantic relatedness.

\smallskip\noindent\textbf{Graph Similarity and Matching.}
Our setting is loosely connected to graph similarity and matching, which compares structured objects via graph edit distance (GED) or learned similarity. Representative methods include Graph Matching Networks~\cite{li2019graphmatching} and SimGNN~\cite{bai2019simgnn}. More recent work improves similarity prediction via hierarchical matching or explicit alignment regularization, including H2MN~\cite{zhang2021h2mn} and ERIC~\cite{zhuo2022eric}, and GED formulations such as GREED~\cite{ranjan2022greed} and GraphEdX~\cite{jain2024graph}. These methods target graph similarity, while overlap ratio estimation is defined on tables with row--column structure and supervision from the overlap ratio.

\section{Conclusion and Future Work}
We study the table overlap ratio estimation problem. Existing estimators remain limited by underexpressed joint \revfour{row--column structure}, missing inter-table alignment signals under independent encoding, and sensitivity to value distribution shift across domains. We propose \Our{}, which addresses these challenges with a structure-aware hypergraph representation, inter-table alignment-guided learning, and domain-robust value handling. Experiments across multiple datasets show that \Our{} improves accuracy, especially in zero-shot transfer, while remaining efficient. Future work includes extending overlap beyond exact matching to tolerate noise and semantic equivalence, developing stronger retrieval-oriented training and indexing for query-by-table at larger scales, and reducing label dependence via weak supervision or self-training to cover new domains.

\clearpage
\bibliographystyle{ACM-Reference-Format}
\bibliography{references}

\clearpage
\appendix
\section*{Appendix}

\section{Additional Details for Query-by-Table Retrieval}\label{app:exp}

\begin{figure}[H]
    \centering
    \includegraphics[width=0.48\textwidth]{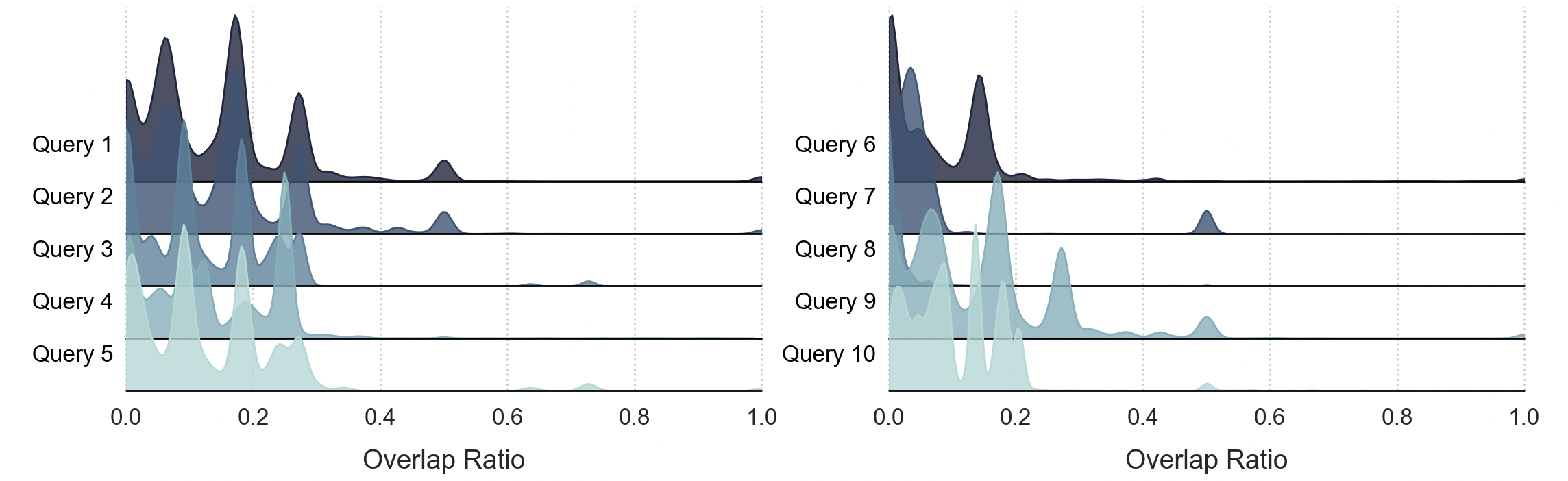}
    \caption[]{Kernel density estimation of overlap ratio distributions for 10 representative query tables on \gitq{}.}
    \label{fig:query_olap_dist}
    \vspace{-3em}
\end{figure}
\begin{figure}[H]
    \centering
    \includegraphics[width=0.48\textwidth]{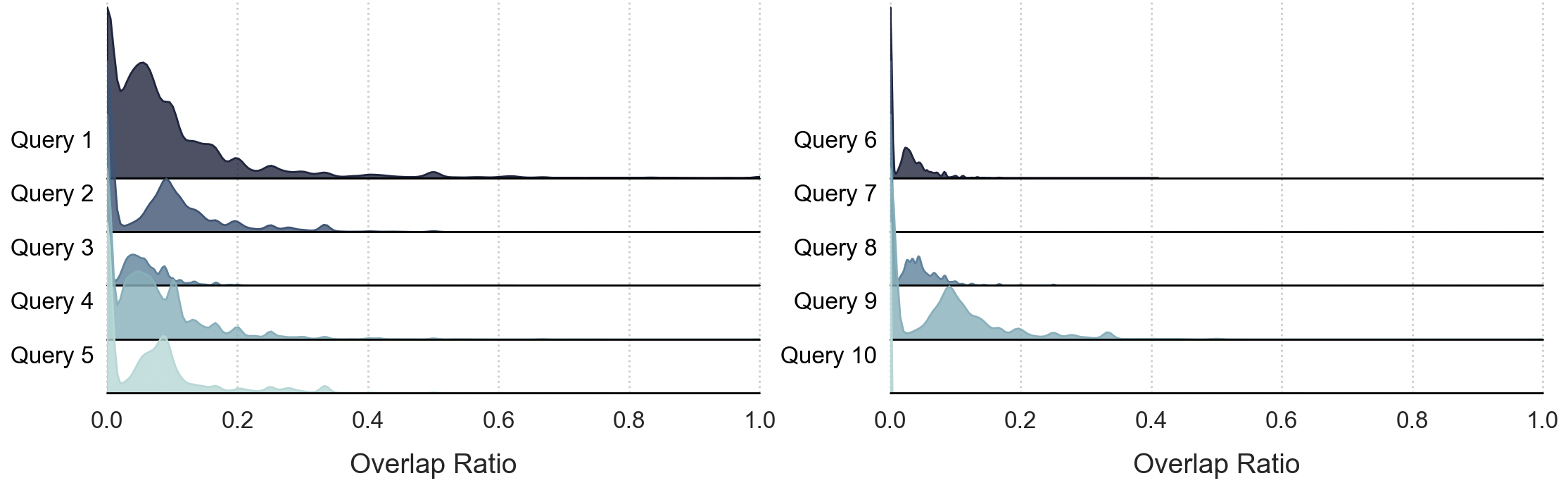}
    \caption[]{Kernel density estimation of overlap ratio distributions for 10 representative query tables on TP-TR.}
    \label{fig:query_olap_dist_tptr}
    \vspace{-1em}
\end{figure}

\subsection{Data Distributions}\label{app:exp-retrieval}

Figures~\ref{fig:query_olap_dist} and~\ref{fig:query_olap_dist_tptr} visualize the ground-truth overlap ratio distributions between each query and its candidates for \gitq{} and TP-TR datasets (10 representative queries shown). In both datasets, overlaps are highly concentrated between 0 and 0.3, with only a small tail of moderate-to-high overlap. This reflects the typical large-corpus setting where most tables are dissimilar, and retrieval performance depends on separating a few truly overlapping tables from many near-zero candidates.

\subsection{Experimental Results on TP-TR Dataset}\label{app:tptr}
We additionally evaluate query-by-table retrieval on the TP-TR (small) benchmark~\cite{fan2024gent}, following the setup where the query set contains 26 TP-TR tables and the candidate set consists of \wiki{} plus 32 additional TP-TR tables. For each query, each method estimates overlap ratios for all candidates and uses the scores for both ranking and threshold-based filtering. Table~\ref{tab:rank_ndcg_tptr} reports nDCG@$k$ at multiple cutoffs, where \our{} achieves the best ranking quality across all $k$, consistently improving over the baselines, which indicates that \our{} better prioritizes highly overlapping tables at the top of the ranked list. Figure~\ref{fig:query_olap_sel_tptr} further evaluates threshold-based retrieval by varying $\tau \in \{0.6, 0.7, 0.8, 0.9\}$ and reporting selection rate, recall, precision, and F1. Overall, \our{} provides the strongest effectiveness-cost trade-off, maintaining higher recall, precision, and F1 while keeping the selected candidate subset compact, whereas baselines including \bertT{} and \turl{} typically return slightly larger subsets to achieve similar recall and \jbn{} can be overly conservative on this dataset.

\begin{table}[H]
    \centering
    \footnotesize
    
    \caption{Ranking performance on TP-TR measured by nDCG@k ($\uparrow$).}
    \label{tab:rank_ndcg_tptr}
    
    \begin{threeparttable}
        \begin{tabular}{l c c c c}
            \toprule
            
            \textbf{Method} & \textbf{nDCG@1} & \textbf{nDCG@10} & \textbf{nDCG@50} & \textbf{nDCG@100} \\
            
            \midrule
            
            \js{}      & 0.720 & 0.609 & 0.499 & 0.476 \\
            \jbu{}   & 0.652 & 0.596 & 0.551 & 0.541 \\
            \jbn{}   & \hlSec{0.781} & \hlSec{0.777} & \hlSec{0.746} & \hlSec{0.753} \\
            \bertR{}    & \sdval{0.577}{0.025} & \sdval{0.560}{0.004} & \sdval{0.488}{0.002} & \sdval{0.453}{0.003} \\
            \bertT{}    & \sdval{0.542}{0.009} & \sdval{0.515}{0.007} & \sdval{0.468}{0.006} & \sdval{0.425}{0.006} \\
            \bertHT{} & \sdval{0.521}{0.011} & \sdval{0.495}{0.006} & \sdval{0.440}{0.004} & \sdval{0.412}{0.008} \\
            
            \robR{} & \sdval{0.594}{0.021} & \sdval{0.572}{0.026} & \sdval{0.525}{0.029} & \sdval{0.499}{0.032} \\
            \robT{} & \sdval{0.588}{0.021} & \sdval{0.565}{0.026} & \sdval{0.518}{0.029} & \sdval{0.495}{0.032} \\
            \robHT{} & \sdval{0.560}{0.021} & \sdval{0.540}{0.026} & \sdval{0.492}{0.029} & \sdval{0.475}{0.032} \\
            
            \turl{}     & \sdval{0.615}{0.021} & \sdval{0.585}{0.026} & \sdval{0.535}{0.029} & \sdval{0.510}{0.032} \\
            \arma{}    & \sdval{0.735}{0.066} & \sdval{0.702}{0.025} & \sdval{0.665}{0.010} & \sdval{0.630}{0.018} \\
            \textbf{\our{}} & \hlBest{\sdval{0.792}{0.021}} & \hlBest{\sdval{0.789}{0.026}} & \hlBest{\sdval{0.765}{0.029}} & \hlBest{\sdval{0.772}{0.032}} \\
            
            \bottomrule
        \end{tabular}
    \end{threeparttable}
    \vspace{-1em}
\end{table}

\begin{figure}[H]
    \centering
    \includegraphics[width=0.48\textwidth]{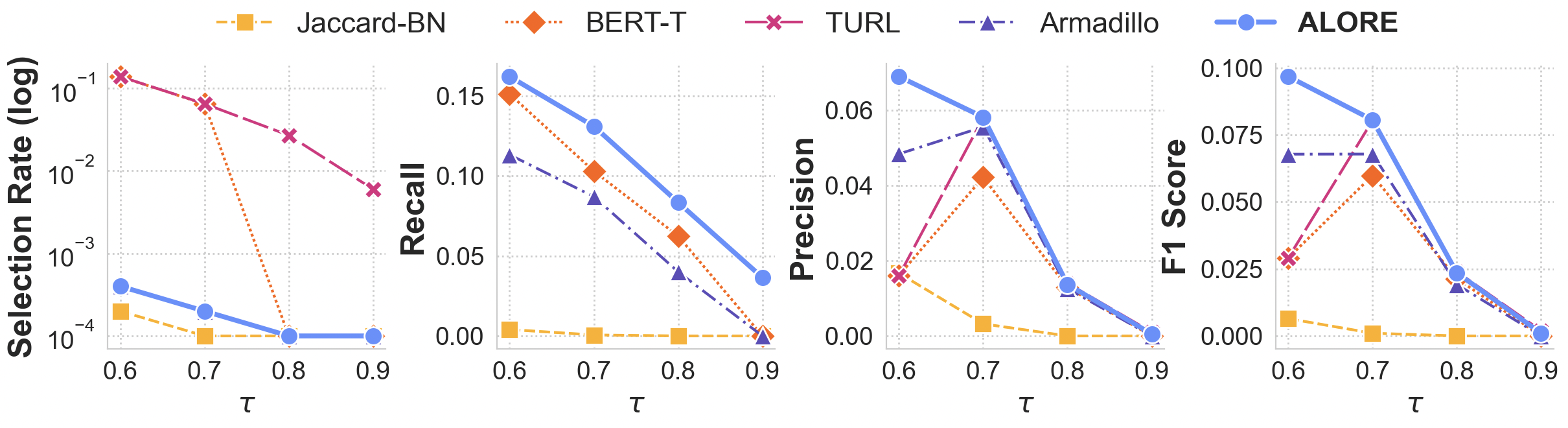}
    \caption[]{Threshold-based retrieval performance on TP-TR across overlap ratio thresholds $\tau$. We report Selection Rate ($\downarrow$), Recall ($\uparrow$), Precision ($\uparrow$), and F1 Score ($\uparrow$).}
    \label{fig:query_olap_sel_tptr}
    \vspace{-1em}
\end{figure}




\end{document}